\RequirePackage{snapshot}
\documentclass{article}
\usepackage{arxiv_iclr26,times}

\usepackage{amsmath,amsfonts,bm}

\def\eqref#1{equation~\ref{#1}}

\def\1{\bm{1}}

\DeclareMathAlphabet{\mathsfit}{\encodingdefault}{\sfdefault}{m}{sl}
\SetMathAlphabet{\mathsfit}{bold}{\encodingdefault}{\sfdefault}{bx}{n}

\newcommand{\E}{\mathbb{E}}

\iclrfinalcopy
\usepackage[hidelinks]{hyperref}
\hypersetup{
    colorlinks=true,
    citecolor=cyan,
    linkcolor=blue
}
\usepackage[capitalize,noabbrev]{cleveref}
\usepackage{url}
\usepackage{booktabs}
\usepackage{amsfonts}
\usepackage{amssymb}
\usepackage{amsthm}
\usepackage{nicefrac}
\usepackage{microtype}
\usepackage{xcolor}
\usepackage{enumitem,kantlipsum}
\crefname{section}{$\mathsection$}{$\mathsection\mathsection$}
\Crefname{section}{$\mathsection$}{$\mathsection\mathsection$}
\crefname{figure}{Fig.}{Fig.}
\Crefname{figure}{Figure}{Figures}
\usepackage[normalem]{ulem}
\usepackage{bbm}

\usepackage{array}
\newcolumntype{C}[1]{>{\centering\arraybackslash}m{#1}}
\newcolumntype{H}{>{\setbox0=\hbox\bgroup}c<{\egroup}@{}}
\usepackage{pifont}
\usepackage{xspace}
\usepackage{enumitem}
\usepackage{wrapfig}
\usepackage{subcaption}
\usepackage{multirow}
\usepackage{tikz}
\newcommand{\circone}{\ding{172}\xspace}
\newcommand{\circtwo}{\ding{173}\xspace}
\newcommand{\circthree}{\ding{174}\xspace}

\newtheorem{proposition}{Proposition}[section]

\usepackage{csquotes}

\newcommand{\algname}{Exploratory Annealed Decoding\xspace}
\newcommand{\alg}{EAD\xspace}

\usepackage{tcolorbox} 
\tcbset{
    colback=yellow!10,
    colframe=black,
    boxrule=0.5mm,
    arc=4mm,
    auto outer arc,
    left=2mm,
    right=2mm,
    top=2mm,
    bottom=2mm,
}

\title{Let it Calm: Exploratory Annealed Decoding for Verifiable Reinforcement Learning}

\author{
Chenghao Yang\textsuperscript{1*},~~~~~
Lin Gui\textsuperscript{2*},~~~~~
Chenxiao Yang\textsuperscript{3*}, \\
~\textbf{Victor Veitch\textsuperscript{2,4},}~~~~~~~~
\textbf{Lizhu Zhang\textsuperscript{5},}~~~~~
\textbf{Zhuokai Zhao\textsuperscript{5}} \\
\textsuperscript{1}Department of Computer Science, University of Chicago ~ \\
\textsuperscript{2}Department of Statistics, University of Chicago ~ \\
\textsuperscript{3}Toyota Technological Insitute at Chicago ~\\
\textsuperscript{4}Data Science Institute, University of Chicago ~
\textsuperscript{5}Meta AI~ \\
\texttt{\{chenghao, glin6\}@uchicago.edu}, 
\texttt{chenxiao@ttic.edu}, \\
\texttt{zhuokai@meta.com}\\
\textsuperscript{*}Equal Contribution \hspace{3mm}\parbox{0.03\textwidth}{\includegraphics[width=\linewidth]{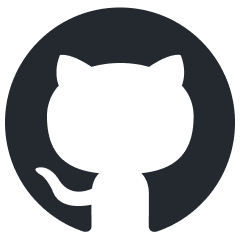}}\hspace{0.5mm}\href{https://github.com/yangalan123/EAD-RLVR}{\hspace{1mm}\texttt{Codebase} } 
\parbox{0.03\textwidth}{\includegraphics[width=\linewidth]{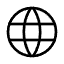}}\hspace{0.5mm}\href{https://yangalan123.github.io/ead_rlvr/}{\hspace{1mm}\texttt{Website} }
}

\begin{document}

\maketitle

\begin{abstract}
Reinforcement learning with verifiable rewards (RLVR) is a powerful paradigm for enhancing the reasoning capabilities of large language models (LLMs), yet its success hinges on effective exploration. An ideal exploration strategy must navigate two fundamental challenges: it must preserve sample quality while also ensuring training stability. While standard fixed-temperature sampling is simple, it struggles to balance these competing demands, as high temperatures degrade sample quality and low temperatures limit discovery. In this work, we propose a simpler and more effective strategy, \algname (\alg), grounded in the insight that exploration is most impactful on early tokens which define a sequence's semantic direction. \alg implements an intuitive \textit{explore-at-the-beginning, exploit-at-the-end} strategy by annealing the sampling temperature from high to low during generation. 
This dynamic schedule encourages meaningful, high-level diversity at the start, then gradually lowers the temperature to preserve sample quality and keep the sampling distribution close to the target policy, which is essential for stable training. 
We demonstrate that \alg is a lightweight, plug-and-play method that significantly improves sample efficiency, consistently outperforming fixed-temperature sampling across various RLVR algorithms and model sizes. 
Our work suggests that aligning exploration with the natural dynamics of sequential generation offers a robust path to improving LLM reasoning.
\end{abstract}

\section{Introduction}
\label{sec:introduction}

Reinforcement learning with verifiable rewards (RLVR) is a powerful approach to enhance the capabilities of Large Language Models (LLMs) in domains such as mathematical reasoning and code generation~\citep{openai2024learning, guo2025deepseek, team2025kimi, yang2025qwen3}. 
In this framework, an LLM learns by iteratively generating potential solutions (i.e., rollouts), and receiving feedback on its attempts. 
A central challenge lies in guiding language models to explore diverse yet high-quality solutions in their vast output space~\citep{cheng2025reasoning};
this reflects the long-standing hard trade-off between exploration and exploitation in RL~\citep{thrun1992efficient, sutton1998reinforcement}.

To achieve effective exploration, the sampling process can be modified to increase the variance of its underlying distribution.
However, any such modification involves two fundamental challenges.
First, it must \textbf{preserve sample quality}. 
Increasing diversity at the cost of generating low-quality, nonsensical outputs is counterproductive.
Second, it must \textbf{ensure training stability}.
Modifying the sampler creates a discrepancy between the behavior policy (used for sampling) and the target policy (being optimized), which necessitates an importance sampling (IS) correction in the gradient update~\citep{degris2012off}.
If the probability ratio in the IS weight is too large, the gradients can have high variance, destabilizing the entire training process~\citep{schulman2017proximal}.
An ideal exploration technique must therefore increase diversity while keeping the sampling distribution close enough to the target policy to allow for stable learning~\citep{haarnoja2018soft, ziegler2019fine}.

A widely adopted and principled way to this trade-off is to adjust the sampling temperature~\citep{ACKLEY1985147, hou2025t1}.
Beyond its implementation simplicity, this method is \emph{variationally optimal}, that is, maximizing entropy (thereby increasing diversity) while bounding the KL divergence from the target policy~\citep{jaynes1957information}.
However, relying on a single fixed temperature creates a tension:
high temperature promotes diversity but produces nonsensical text~\citep{renze-2024-effect, wang2025beyond},
whereas low temperature improves quality but limits exploration, leading to generic and repetitive outputs~\citep{Holtzman2020The, guo2025deepseek}.

In this work, we propose \textbf{\algname (\alg)}, a strategy that improves the balance of this trade-off by leveraging a key insight into sequential generation: exploration is not equally valuable at every step.
The initial tokens shape a sequence's semantic direction and structure, making early exploration crucial for discovering diverse valid solutions. 
Later tokens, however, fill in details within the established context, where excessive exploration can harm coherence.
This insight motivates our core strategy: \textit{explore at the beginning, exploit at the end}. 
This simple principle elegantly addresses the twin challenges of quality and stability. 
Injecting randomness early promotes diverse, high-level exploration, while reducing it later ensures completions are both coherent and close to the target policy---an essential property for stable off-policy learning.

In summary, our contributions are as follows:
\begin{enumerate}[wide, labelwidth=!, labelindent=0pt]
    \item[\circone] We propose \alg, a simple and effective exploration strategy for RLVR that dynamically anneals temperature to encourage meaningful diversity while maintaining high sample quality.
    \item[\circtwo] We show \alg is a plug-and-play enhancement that improves sample efficiency over temperature sampling, delivering robust gains across various RLVR algorithms including GRPO~\citep{shao2024deepseekmath}, DAPO~\citep{yu2025dapo}, and EntropyMech~\citep{cui2025entropy} on both small and larger models.
    \item[\circthree] We show that \alg can be adapted for test-time inference, where a tuned temperature schedule further enhances generation quality.
\end{enumerate}

\section{Preliminary}\label{sec:prelim}
\paragraph{Notations.} 

Let $x$ be a prompt from a dataset $\mathcal{D}$, and let $y = (y_1, \dots, y_{|y|})$ be a generated response sequence, where $y_{<t}$ denotes the prefix $(y_1, \dots, y_{t-1})$. We define an LLM as a policy $\pi_\theta$ parameterized by $\theta$, and denote the reference policy, i.e., the starting point for RL, as $\pi_{\mathrm{ref}}$. The probability of generating $y$ given $x$ is defined autoregressively as $\pi_\theta(y \mid x) = \prod_{t=1}^{|y|} \pi_\theta(y_t \mid [x, y_{<t}])$. A reward model $R(x, y)$ evaluates the quality of a prompt-response pair; for our RLVR experiments, we use the rule-based Math-Verify reward model.\footnote{\url{https://github.com/huggingface/Math-Verify}} Finally, we use $|\cdot|$ to denote the length of a sequence or the cardinality of a set, and use the shorthand $1:n$ for the set $\{1, \dots, n\}$.

\paragraph{Reinforcement Learning with Verifiable Rewards (RLVR).}
The standard objective for Reinforcement fine-tuning of LLMs is to maximize the expected reward over a prompt dataset $\mathcal{D}$:
\begin{equation}
\label{eq:rl-main-objective}
    \max_\theta J(\theta) := \E_{x \sim \mathcal{D}, y \sim \pi_\theta(\cdot\mid x)}\left[ R(x, y) \right].
\end{equation}
However, on complex reasoning tasks, learned reward models $R(x,y)$ are prone to \emph{reward hacking}, where they assign high scores to plausible but incorrect solutions \citep{gao2023scaling, perez2023discovering, weng2024rewardhack, wcq2025beyond}. RLVR addresses this by replacing the learned model with a verifiable, rule-based reward signal—such as a verifier that provides binary feedback on a solution's correctness \citep{guo2025deepseek}. This ensures that the policy is optimized using a reliable signal.

RLVR is commonly implemented using policy gradient algorithms like Proximal Policy Optimization (PPO) \citep{schulman2017proximal}. However, standard PPO often requires complex token-level advantage estimation and a separate value model. To better suit RLVR's trajectory-level binary rewards, subsequent methods simplify the advantage calculation \citep{shao2024deepseekmath,yu2025dapo}. A prominent example is Decoupled Clip and Dynamic Sampling Policy Optimization (DAPO)~\citep{yu2025dapo}, which optimizes:
{
\begin{align*}
&J_{\text{DAPO}}(\theta)=\\ 
&\E_{x \sim \mathcal{D},y^{(1:G)} \overset{iid}{\sim} \pi_{\theta_{\text{old}}}(\cdot\mid x)} 
\left[\tfrac{1}{\sum_{i=1}^G |y^{(i)}|}\sum_{i=1}^G\sum_{t=1}^{|y^{(i)}|}\min\left\{r^{(i)}_{t}(\theta)A_i,\mathrm{clip}\left(r^{(i)}_{t}(\theta),1-\varepsilon_{\text{low}},1+\varepsilon_{\text{high}}\right)A_i\right\}
\right]\\
&\mathrm{s.t.}\quad 0<\left|\{y^{(i)}:y^{(i)}\mathrm{~is~correct}\}\right|<G,
\end{align*}
}where $\pi_{\theta_{\text{old}}}$ refers to previous policy and $r^{(i)}_{t}(\theta)=\frac{\pi_{\theta}(y^{(i)}_{t} \mid [x, y^{(i)}_{<t}])}{\pi_{\theta_{\text{old}}}(y^{(i)}_{t} \mid [x, y^{(i)}_{<t}])}$.  The asymmetric bound $\varepsilon_{\text{high}} > \varepsilon_{\text{low}}$ is proposed to relax the restriction on probability increase and encourage more exploration in the training. 
The advantage $A_i$ is computed by normalizing the binary rewards across a batch of $G$ responses, thus avoiding the need for a value model:
\begin{small}
\begin{equation*}
A_i=\frac{R_i-\mathrm{mean}_{k\in 1:G}(R_k)}{\mathrm{std}_{k\in 1:G}(R_k)},~
\text{where } R_i=R(x,y^{(i)}).
\end{equation*}
\end{small}

\paragraph{Temperature.}
Temperature sampling \citep{ACKLEY1985147} is a widely used method to control the stochasticity of the policy $\pi_\theta$. At each generation step $t$, the LLM computes a vector of logits, $\mathbf{h}$, over the vocabulary $V$ based on the prompt $x$ and the preceding tokens $y_{<t}$. Temperature sampling rescales these logits with a parameter $\tau > 0$ before applying the softmax function to form the next-token probability distribution:
$$
\pi_\theta(y_t = v \mid [x, y_{<t}]; \tau) = \frac{\exp(h_v / \tau)}{\sum_{v' \in V} \exp(h_{v'} / \tau)},
$$
where $v$ is a token in the vocabulary $V$ and $h_v$ is its corresponding logit. The temperature $\tau$ directly modulates the sharpness of the output distribution. A higher temperature ($\tau > 1$) flattens the distribution, increasing output diversity by making less likely tokens more probable. Conversely, a lower temperature ($\tau < 1$) sharpens it, leading to more deterministic, greedy outputs. 

\paragraph{Pass@$\bm k$.} 
Pass@$k$ measures the probability that at least one of $k$ independent outputs from a language model is correct. 
Let $p_x$ denote the underlying accuracy of the language model $\pi$ given one prompt $x$. The pass@$k$ accuracy is defined as
\[
\E_{x\sim\mathcal{D}}\left[1-(1-p_x)^k\right].
\]
The inner term $1-(1-p_x)^k$ quickly approaches one unless $p_x$ is near zero. 
For example, when $k=64$, any $p_x \geq 0.0695$ already yields a probability of at least $0.99$. 
A large gap between pass@1 and pass@$k$ suggests that for some prompts, $p_x$ is essentially zero.
High diversity may benefit this metric because a model with diverse outputs is more likely to maintain non-negligible $p_x$ values across prompts, leading to stronger pass@$k$ performance.

\paragraph{Entropy.} 
Given a prompt $x$ and a prefix $y_{<t}$, the \textbf{token-level entropy} at step $t$ is defined over the policy's conditional distribution:
$$
H(Y_t \mid [x, y_{<t}]; \theta) := -\sum_{v \in V} \pi_\theta(y_t = v \mid [x, y_{<t}]) \log \pi_\theta(y_t = v \mid [x, y_{<t}]),
$$
where the sum is over all tokens $v$ in the vocabulary $V$.

The \textbf{average entropy} of the policy, $H(\pi_\theta)$, is the expected token-level entropy over all prompts and generation steps. We can estimate this value empirically using Monte Carlo sampling. For each prompt $x \in \mathcal{D}$, we generate $G$ i.i.d. responses $y^{(1)}, \dots, y^{(G)}$. The average entropy is then approximated by averaging the token-level entropies across all generated tokens:
$$
\bar{H}(\pi_\theta) \approx \frac{1}{|\mathcal{D}|} \sum_{x \in \mathcal{D}} \left( \frac{\sum_{i=1}^G \sum_{t=1}^{|y^{(i)}|} H(Y_t \mid [x, y^{(i)}_{<t}]; \theta)}{\sum_{i=1}^G |y^{(i)}|} \right).
$$

\section{Sequential Exploration: Explore Early, Exploit Late}
\label{sec:exploitation-vs-exploration}
Exploration is a cornerstone of reinforcement learning, enabling agents to discover high-quality policies rather than settling on suboptimal solutions~\citep{sutton1998reinforcement, ladosz2022exploration}. 
This principle becomes particularly vital in deep RL, where vast action spaces render exhaustive search infeasible.
In the context of RL for language models (e.g., RLVR), insufficient exploration often manifests as  \textit{entropy collapse}, i.e., a premature narrowing of the generation distribution during training~\citep{yu2025dapo, wang2025beyond, cui2025entropy}.
A common simple tool to encourage exploration is \emph{temperature sampling}.
However, a \textit{fixed} temperature imposes a difficult trade-off. 
A high temperature promotes diversity (as indicated by increased entropy\footnote{See \cref{app:entropy_proof} for the proof.}), but it risks degrading output quality with nonsensical tokens and hallucinations~\citep{renze-2024-effect, wang2025beyond}. 
In contrast, a low temperature limits the discovery of novel solutions, leading to generic and repetitive outputs~\citep{Holtzman2020The, guo2025deepseek}.

\begin{wrapfigure}{r}{0.35\textwidth}
\vspace{-0.25in}
    \centering
    \includegraphics[width=\linewidth]{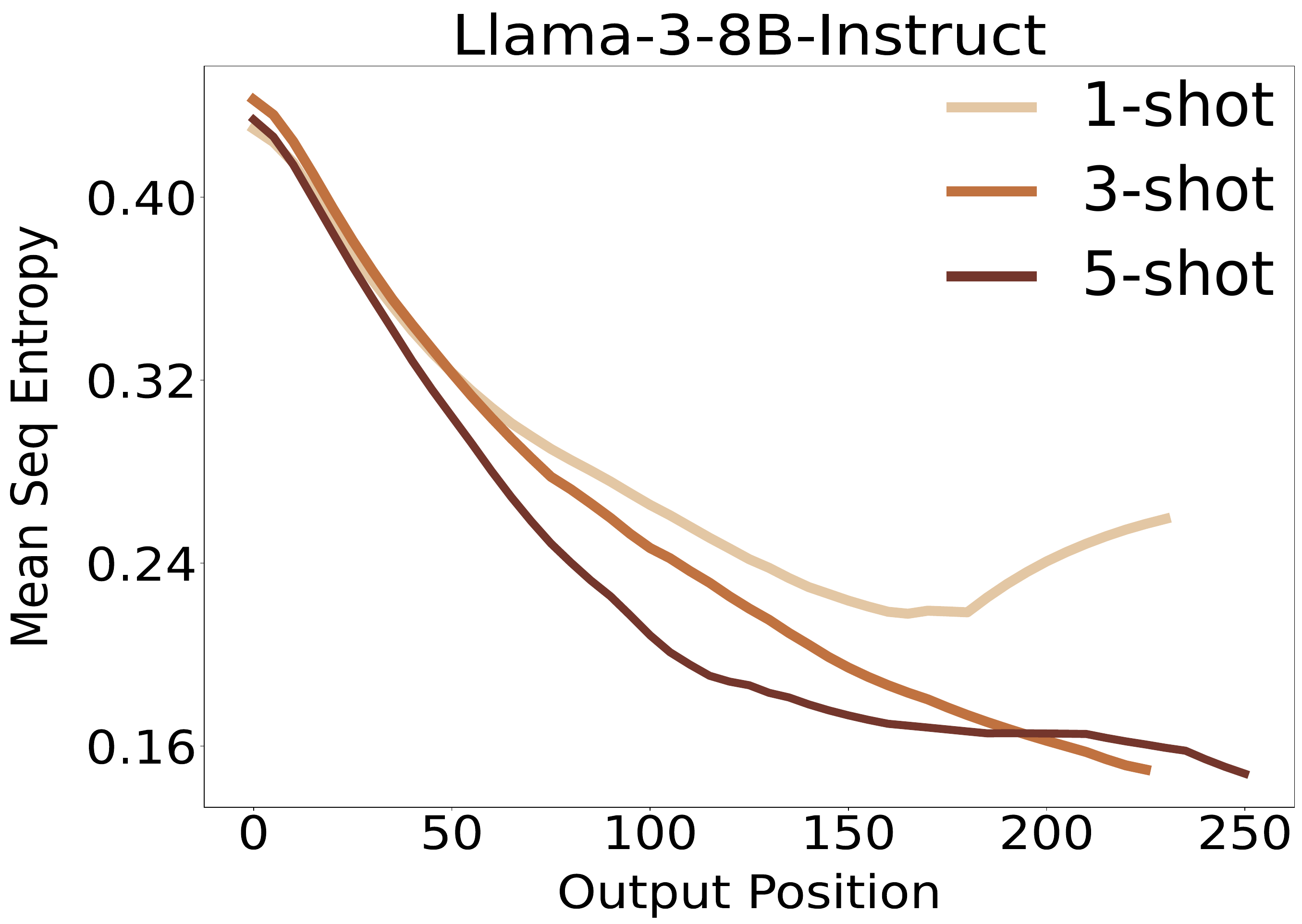}
    \vspace{-0.3in}
    \caption{
    Average entropy shrinks with output positions for Llama-3-8B-Instruct on MMLU dataset. 
    }
    \label{fig: entropy_sequential_dynamic}
\end{wrapfigure}
The key to resolving this dilemma lies not in finding a single best temperature, but in recognizing that exploration requirements vary throughout the generation process.
This key insight stems directly from the autoregressive nature of language models.
At the beginning of a sequence, the context is minimal and uncertainty is high, allowing a wide range of valid continuations. 
As more tokens are produced, the context becomes increasingly specific, constraining subsequent choices.

This intuition is supported by information theory: the data processing inequality~\citep{shannon1948mathematical} states that expected conditional entropy tends to decrease with each step\footnote{While specific rollouts may have late high-entropy positions, the probability of this is exponentially small with position $t$~\citep{yang2025alignment}, making the overall trend a reliable heuristic.}:
\begin{small}
\begin{align}
\label{eq:entropy_decay}
H(Y_t | [x, Y_{<t}]; \theta) &= \underbrace{\mathbb{E}_{y_{<t}}\big[ H(Y_t | [x, y_{<t}]; \theta) \big]}_{\substack{\text{Expected entropy at step } t \\ \text{(average over all prefixes } y_{<t}\text{)}}} 
\geq 
\underbrace{\mathbb{E}_{y_{<t+1}}\big[ H(Y_{t+1} | [x, y_{<t+1}]; \theta) \big]}_{\substack{\text{Expected entropy at step } t+1 \\ \text{(average over all prefixes } y_{<t+1}\text{)}}} = H(Y_{t+1} | [x, Y_{<t+1}]; \theta)\nonumber
\end{align}
\end{small}We further validate this empirically by examining position-wise entropy trend on the MMLU dataset~\citep{hendrycks2021measuring}\footnote{We use MMLU as a held-out dataset with Chain-of-Thought prompting~\citep{wei2022chain} to incentivize longer reasoning outputs, aligning with a typical RLVR scenario.} with Llama-3-8B-Instruct~\citep{grattafiori2024llama} (see \cref{fig: entropy_sequential_dynamic}).

\begin{wrapfigure}{r}{0.6\textwidth}
\vspace{-0.15in}
    \centering
    \includegraphics[width=\linewidth]{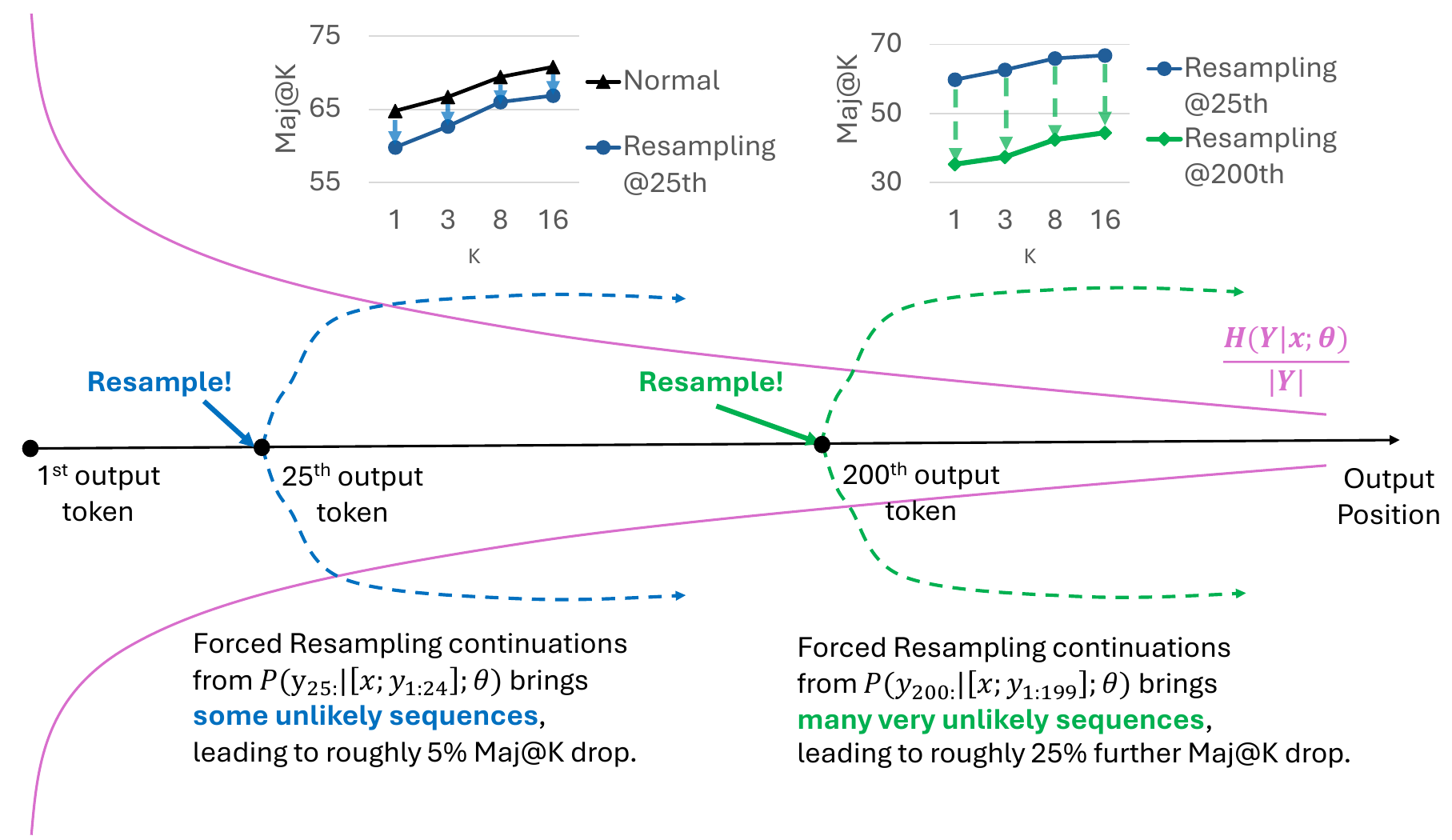}
    \vspace{-0.1in}
    \caption{
       A ``forking'' experiment on DeepSeek-Llama-3 shows early branching (high-entropy region) yields higher Maj@$k$ on MMLU than late branching (low-entropy region).
    }
\label{fig:forking_experiment}
\vspace{-0.05in}
\end{wrapfigure}
This entropy decay has a direct performance implication: effective exploration, producing diverse, high-quality responses, is most beneficial in an uncertain, high-entropy phase at the start of generation.
To verify this, we replicate a controlled ``forking'' experiment~\citep{yang2025alignment} on DeepSeek-Distilled Llama-3-8B~\citep{guo2025deepseek}, a representative RLVR-fine-tuned model from the same Llama-3 family. 
As shown in \cref{fig:forking_experiment}, trajectories branched from early generation steps consistently outperform those branched later. 
This finding is also consistent with observations from inference-time analysis, where forced, late-stage exploration tends to degrade output quality~\citep{liao2025lost, yang2025alignment, fu2025deep}.

Aligning our strategy with the natural dynamics of generation, we arrive at a simple yet powerful design principle: \textbf{explore early and exploit late}.

\section{A Method for Sequential Exploration}
\label{sec: method}
To put the principle of ``explore early, exploit late'' into practice, we introduce \emph{\algname (\alg)}, 
which uses an annealed temperature schedule starting from a higher-than-standard initial temperature (i.e., $\tau>1$). 
To adapt this strategy to RLVR, we further incorporate a \emph{global-step-aware decay rate}, ensuring that the temperature schedule remains effective as the typical response length increases during training.

\paragraph{Exploratory Annealed Decoding.} 
Instead of a fixed temperature, our method dynamically adjusts the temperature $\tau_t$ for each token $t$ in a rollout. The schedule starts at a high temperature $\tau_\mathrm{max} > 1$ and decreases progressively throughout the generation process.
Specifically, we sample the $t$-th token for one rollout 
with the token-level temperature $\tau_t = \max\{1 + \tau_\mathrm{max} - e^{t/d}, \tau_\mathrm{min}\}$,
where we apply the annealed schedule with a \emph{decay rate} $d$ controlling the annealing speed.
\begin{wrapfigure}{r}{0.6\textwidth}
\vspace{-0.2in}
    \centering
    \includegraphics[width=\linewidth]{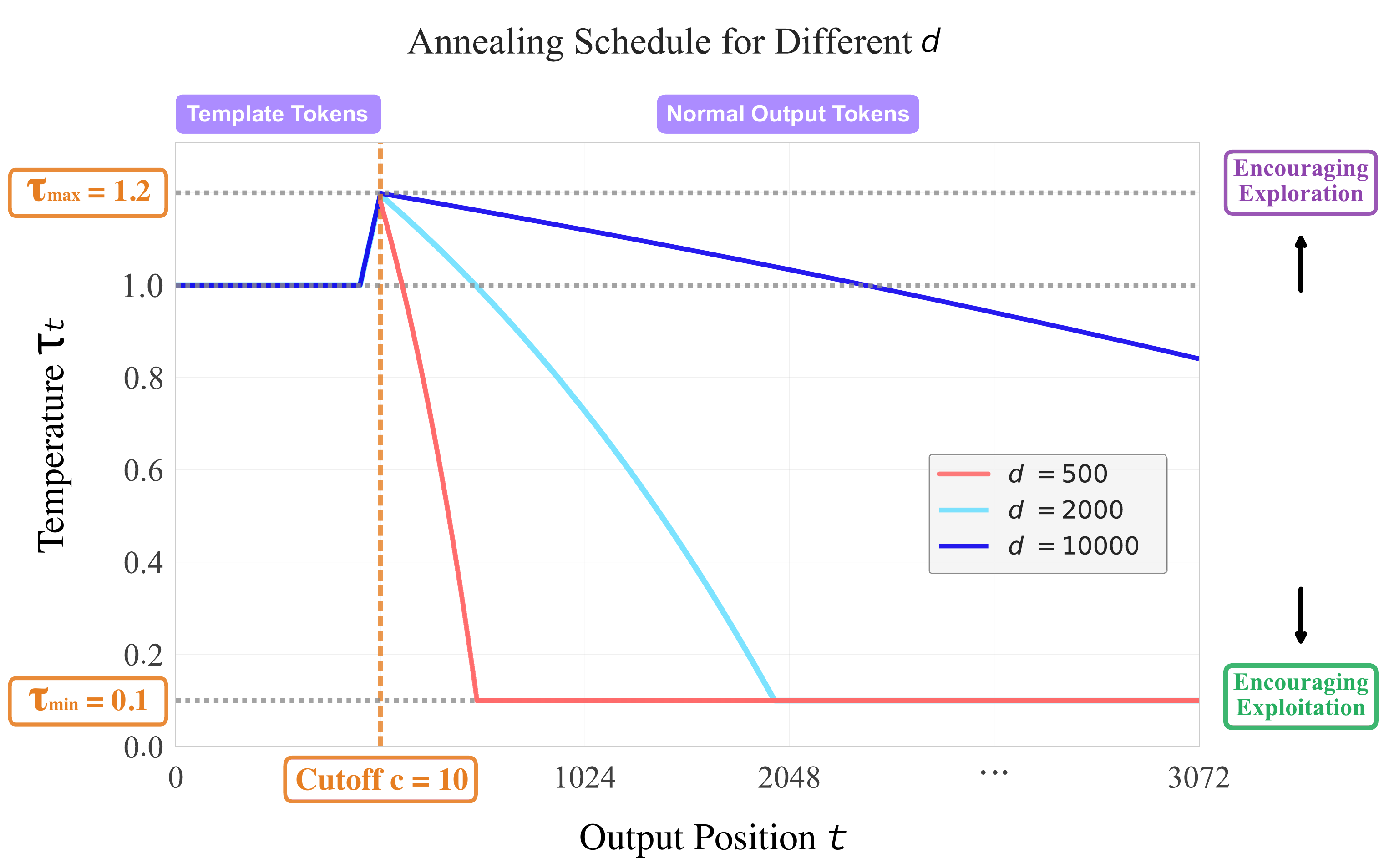}
    \vspace{-0.25in}
    \caption{
     The annealing schedule with different decay rates $d$. A larger $d$ slows the cooling, front-loading exploration over more tokens. We set $c=10, \tau_{\mathrm{max}}=1.2, \tau_{\mathrm{min}}=0.1$ for illustration.
    }
    \label{fig: annealing_schedule_demo}
\vspace{-0.2in}
\end{wrapfigure}
As illustrated in \cref{fig: annealing_schedule_demo}, the decay rate $d$ controls how long the policy remains in a high-exploration state. A larger $d$ front-loads exploration across more initial tokens, while a smaller $d$ transitions to exploitation more quickly. 
In practice, we let $\tau_t = 1.0$ for $t<c$, where $c$ is a pre-determined initial position for the sake of model-specific or prompt-specific template tokens injected in the training process. 
During RLVR, language models tend to generate some template tokens such as ``\textit{let's verify step by steps}'' or repeat the question. 
We fix the temperature at $\tau = 1.0$ in this part to avoid interfering with the generation process.

\paragraph{Global-Step-Aware Decay Rate.}
As training progresses and response lengths increase,\footnote{We illustrate increased length for \alg in \cref{app: increased_length}. } the decay rate $d$ should be adjusted in accordance with the training step.
Otherwise, an excessive number of tokens may be generated under extremely low temperatures, which degrades response quality and leads to undesirable behaviors such as repetition~\citep{guo2025deepseek}, off-topic drift~\citep{spataru2024know}, and unnecessary verbosity~\citep{Holtzman2020The}. 
In particular, we adopt the following \emph{global-step-aware decay rate}:
$d_s = \min(d_0 + 5s, 40000).$

\paragraph{Ensuring Stability with Truncated  Importance Sampling.}
With aggressive annealing schedules (e.g., very small $\tau_\mathrm{min}$ and $d$), sampling low-probability, long-tail tokens can cause the annealed policy to deviate significantly from the one being optimized. This creates an off-policy discrepancy that risks training instability.  To mitigate this, we employ \emph{truncated importance sampling} (TIS)~\citep{heckman1998matching, hilton2022batch, yao2025offpolicy} to correct the objective, ensuring stable optimization even under highly exploratory schedules (see \Cref{sec:off-policy} for details).

Overall, this annealed decoding strategy offers a compelling combination of effectiveness and efficiency. As a plug-and-play modification to standard temperature sampling, it incurs negligible computational overhead and is fully compatible with existing RLVR pipelines and diverse policy optimization algorithms like DAPO, GRPO, etc.

\section{Experiments}\label{sec:experiments}

\subsection{Experimental Setup}
\label{sec: experiment_setup}
\paragraph{Models, Data, and Training Frameworks.}
To ensure a rigorous and controlled comparison, we follow the Minimal-RL recipe~\citep{xiong2025minimalist},\footnote{More training details and hyperparameter setups are illustrated in \cref{app: minimal_rl_training_details}. } training all models on the Numina-Math dataset~\citep{numina_math_7b}, which contains 860k math prompts. To assess the generality of our method, we experiment with both Qwen-2.5-Math-1.5B~\citep{yang2024qwen2} and Llama-3.2-1B-Instruct~\citep{dubey2024llama}.\footnote{We also experimented with the Llama-3.2-1B base model. However, consistent with \citet{wang2025octothinker}, we found that applying RL to base models without intermediate domain-specific fine-tuning yields limited gains across all methods. We defer a deeper investigation to future work.} We also include the larger Qwen-2.5-Math-7B model to evaluate how our approach scales. While our primary experiments are conducted within the DAPO framework~\citep{yu2025dapo}, we demonstrate broader applicability by additionally integrating \alg with GRPO~\citep{shao2024deepseekmath} and EntropyMech~\citep{cui2025entropy}.

\paragraph{Baselines and Controlled Comparison.}
We evaluate \alg against fixed-temperature sampling, a standard and strong baseline, using temperatures $\tau \in \{0.6, 1.0, 1.2\}$ as recommended by prior work~\citep{renze-2024-effect, guo2025deepseek, hou2025t1}. For a fair comparison focused specifically on the sampling strategy, we disable two orthogonal techniques for all methods: (1) dynamic data sampling~\citep{yu2025dapo}, to maintain a consistent training set for all runs, and (2) rollout length penalties, to avoid confounding the reward signal with length-based biases.

\paragraph{Hyperparameters.}
Unless otherwise stated, we use a default configuration of $\tau_{\max}=1.2$ and $d_0=25$ for \alg. 
For $\tau_{\min}$, we observed optimal values varied by model capability.
For the 1B and 1.5B models, we set $\tau_{\min}=0.1$. 
For the more capable 7B model, we found that an overly low temperature could lead it to generate plausible but incorrect solutions; we therefore used a higher value of $\tau_{\min}=0.8$ to mitigate this effect. 
All hyperparameters are tuned based on a prior study over held-out datasets. 

\subsection{\alg Improves RLVR Training}

\paragraph{\alg Improves RL Exploration and Training Efficiency.}
\begin{figure}[t!]
\centering
\includegraphics[width=\linewidth]{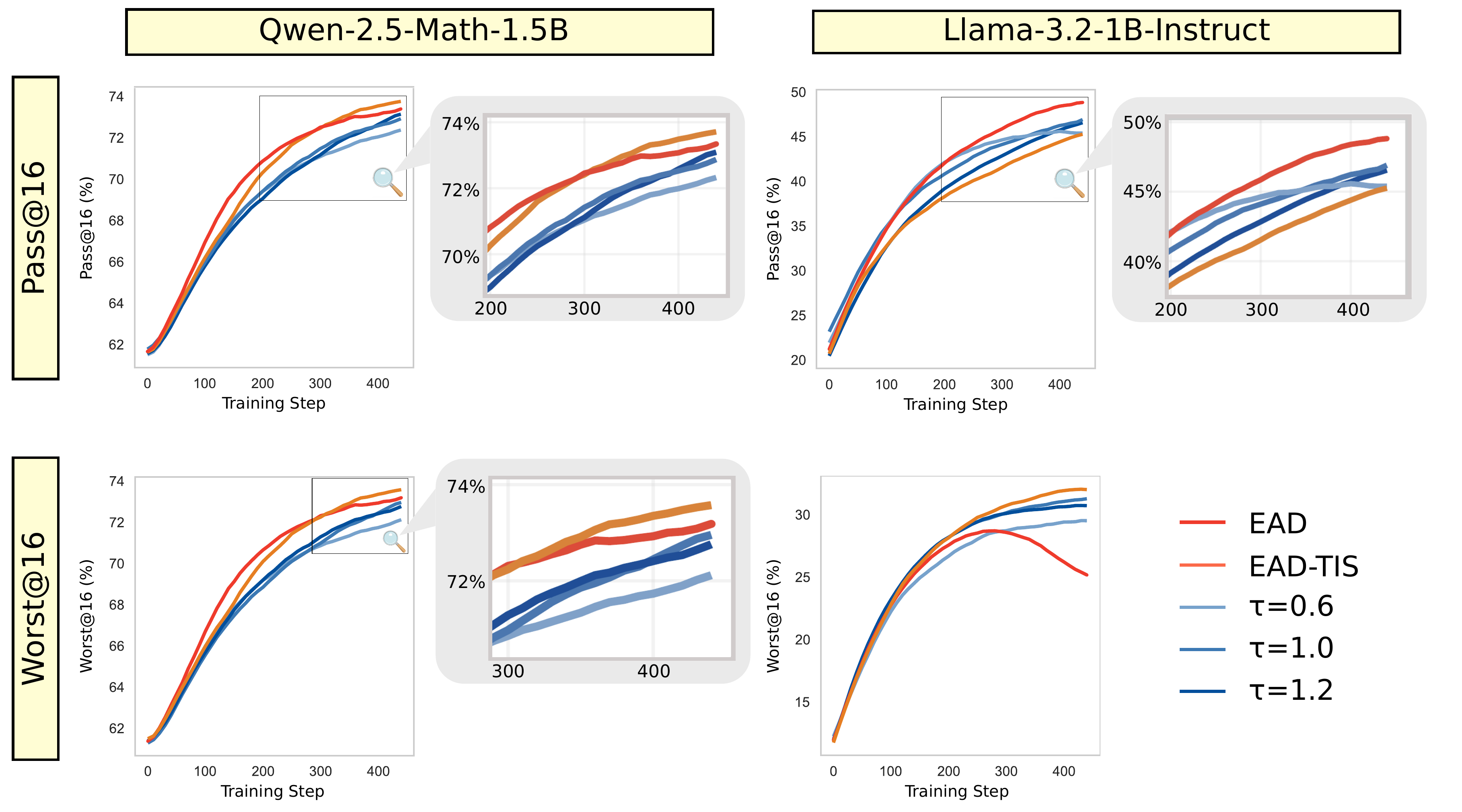}
\caption{
Pass@16 and Worst@16 performance evaluation in RL training. 
While \alg improves exploration of high-quality samples (even the worst outperform temperature sampling), the gain diminishes over time; importance sampling can supplement to correct bias and sustain training.
}
\label{fig: main_bench_best_worst_at_16}
\end{figure}
\begin{wrapfigure}{r}{0.45\textwidth}
\vspace{-0.15in}
    \centering
    \includegraphics[width=\linewidth]{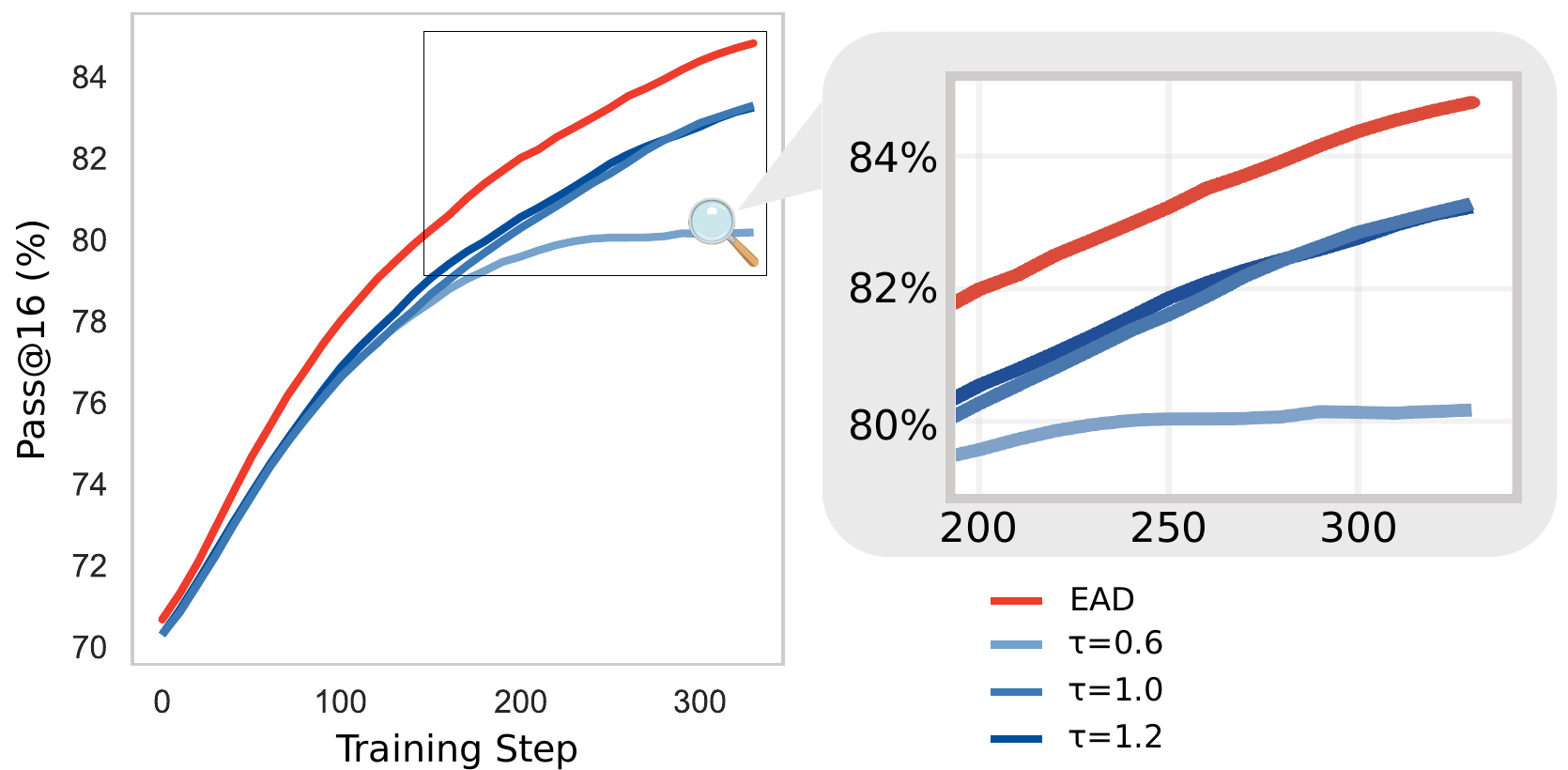}
    \vspace{-0.2in}
    \caption{
    Pass@16 performance on Qwen-2.5-Math-7B. 
    \alg enables better exploration than fixed-temperature sampling, yielding sustained gains in Pass@16 throughout training.
    }
    \label{fig: best_at_16_qwen2.5_7b}
\end{wrapfigure}
As shown in \cref{fig: main_bench_best_worst_at_16}, \alg significantly improves training efficiency. 
For Pass@16 accuracy, 
\alg (w/o TIS) consistently outperforms the baselines on the Llama and Qwen models (\alg (w/ TIS) also outperforms on the Llama model), demonstrating more effective exploration. 
Under the stricter Worst@16 metric, the inclusion of TIS becomes essential for maintaining stable performance gains, highlighting its importance for correcting the off-policy training dynamic introduced by \alg. 
Through bootstrapping evaluation as in~\citet{hochlehnert2025sober}, the standard deviation of both Pass@16 performance and worst@16 are way below 0.01 and thus all comparisons here are significant.

To verify that our method generalizes, we evaluated it on the larger Qwen-2.5-Math-7B model. The results, presented in \cref{fig: best_at_16_qwen2.5_7b}, confirm that the performance gains from \alg remain significant. This demonstrates that our approach is effective not only on smaller models but also scales successfully.

\begin{figure}[t]
    \centering
    \vspace{-0.05in}
    \includegraphics[width=.9\linewidth]{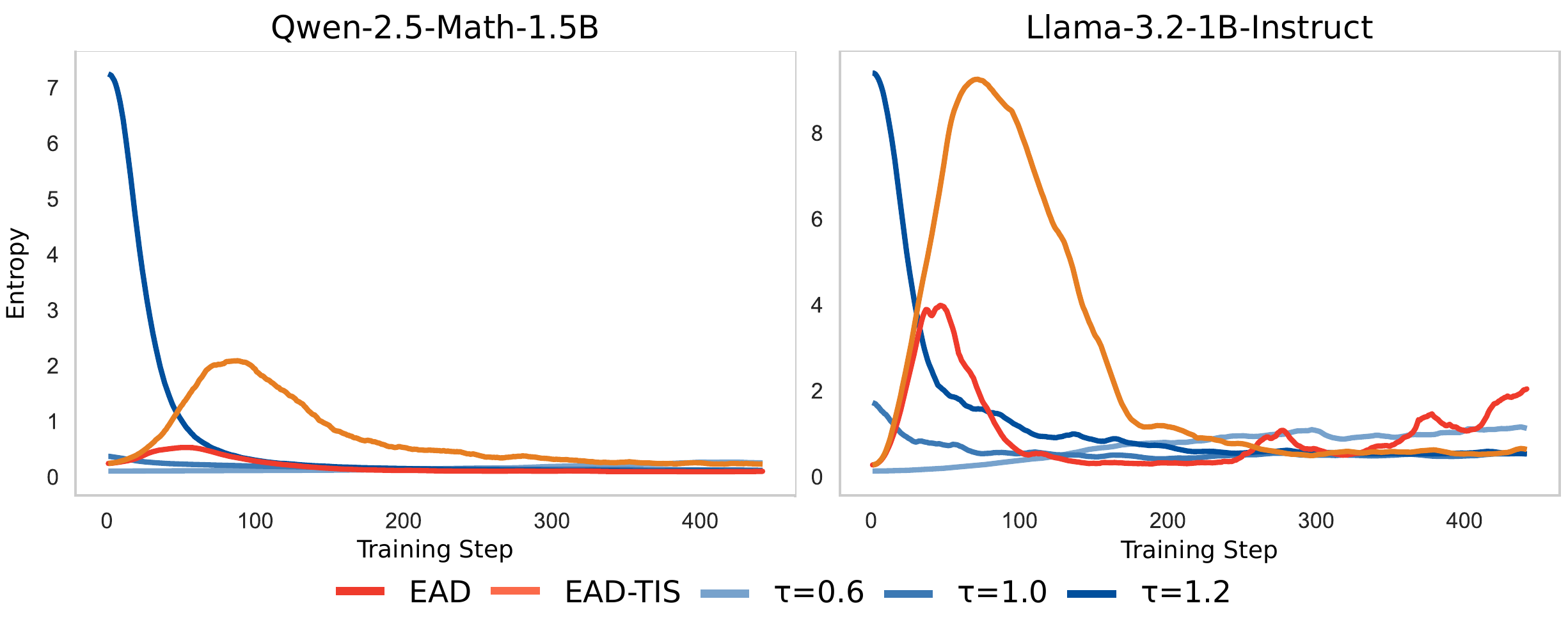}
    \vspace{-0.1in}
    \caption{Entropy Dynamics in RL Training. Under commonly-used temperature sampling, trained with RL algorithm would make entropy decrease, sharply shrinking the exploration space for RL from beginning. While EAD could help RL algorithm to escape local minimum and do exploration when needed in the middle of RL training. 
    }
\label{fig: main_bench_entropy}
\vspace{-0.25in}
\end{figure}

\paragraph{\alg Mitigates Entropy Collapse.}
One major problem in RLVR training is entropy collapse~\citep{cui2025entropy}, which causes the exploration space to shrink and constrains improvement during the ``plateau stage"~\citep{deng2025trial}.
We plot the entropy dynamic in \cref{fig: main_bench_entropy}, where we can see that the entropy dynamic for \alg-empowered methods is not monotonically decreasing from the beginning. 
Instead, it tries to gradually transition out from local optimum~\citep{kirkpatrick1983optimization, bertsimas1993simulated} in a natural, continuous way without any external intervention, such as introducing tree search in rollout sampling~\citep{li2025treepo}.
\begin{wrapfigure}{r}{0.6\textwidth}
    \centering
    \includegraphics[width=\linewidth]{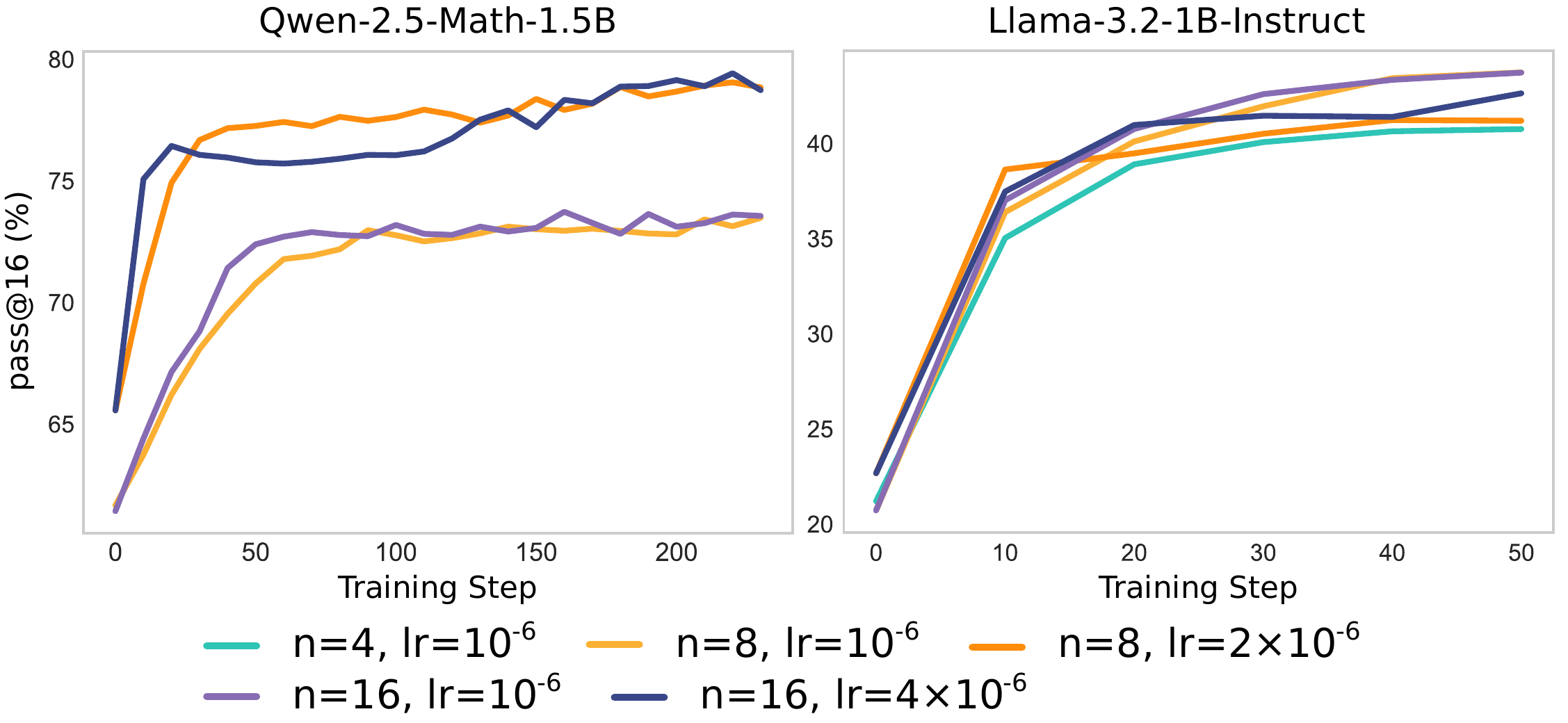}
    \vspace{-0.15in}
    \caption{\alg would bring further performance improvement via increased numbers of rollouts, but the commonly used $4$ or $8$ is already good enough. }
    \label{fig: main_bench_rollout_scaling}
\vspace{-0.4in}
\end{wrapfigure}
\begin{wrapfigure}{r}{0.6\textwidth}
    \vspace{-0.4in}
    \centering
    \includegraphics[width=\linewidth]{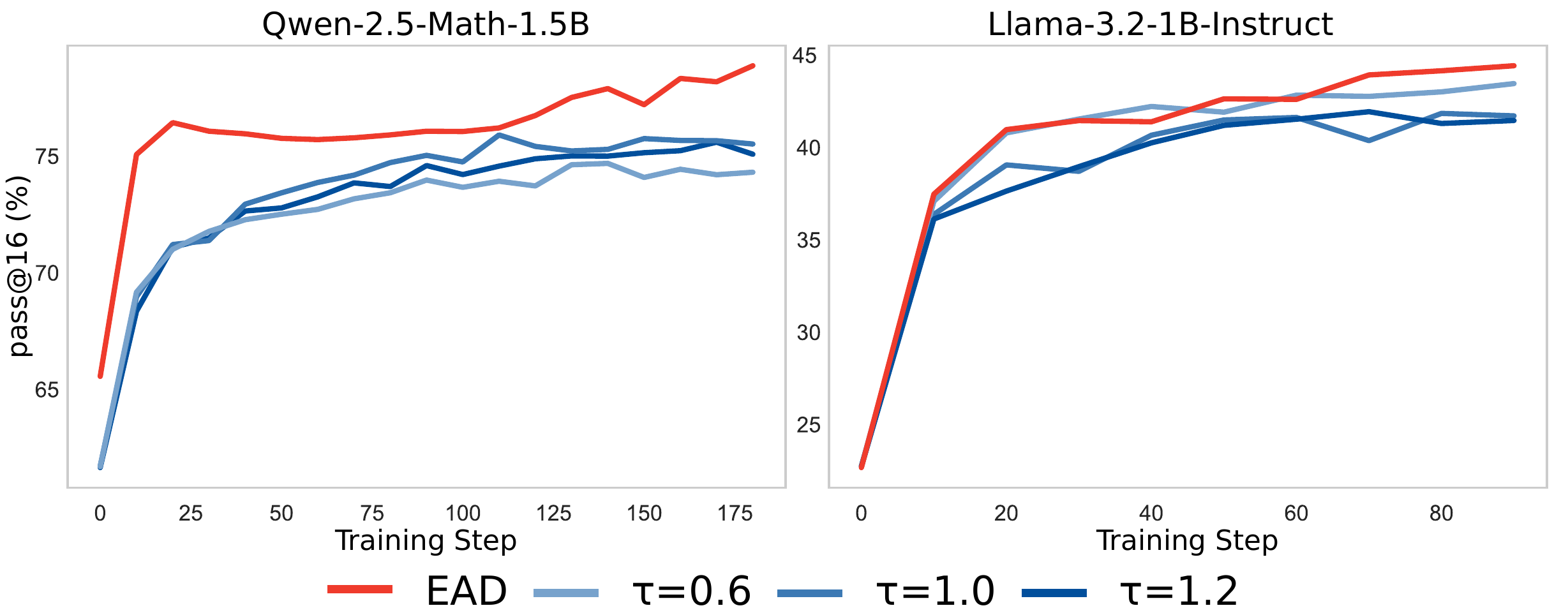}
    \vspace{-0.25in}
    \caption{When scaling out the rollout number to 16, the relative advantages of our methods diminished; however, it still outperforms traditional same-temperature sampling.}
    \label{fig: rollout_modelwise_compare}
    \vspace{-0.3in}
\end{wrapfigure}

\paragraph{Sample efficiency of \alg.}
Increasing the number of rollouts is a common but computationally expensive strategy to enhance exploration~\citep{hou2025t1}. 
We test the sample efficiency of \alg by varying the number of rollouts, adjusting the learning rate accordingly as suggested by prior work~\citep{chen2025pass}. 
As shown in \cref{fig: main_bench_rollout_scaling}, while more rollouts can further improve performance, \alg achieves strong results with just 4 or 8 rollouts. 
For instance, the optimal configurations are $n=8$ with a learning rate of $10^{-6}$ for Llama-3.2-1B-Instruct and $2\times10^{-6}$ for Qwen-2.5-Math-1.5B. 
This highlights the sample efficiency of our approach, offering a way to reduce the computational cost of the rollout phase.

To assess whether \alg's advantage persists with extensive exploration, we compare it against baselines using a larger set of 16 rollouts. 
\cref{fig: rollout_modelwise_compare} shows that although the relative performance gain diminishes, \alg still outperforms fixed-temperature baselines by a clear margin.

\subsection{\alg Improves Inference-Time Scaling}
\begin{wrapfigure}{r}{0.6\textwidth}
\vspace{-0.17in}
    \centering
    \includegraphics[width=\linewidth]{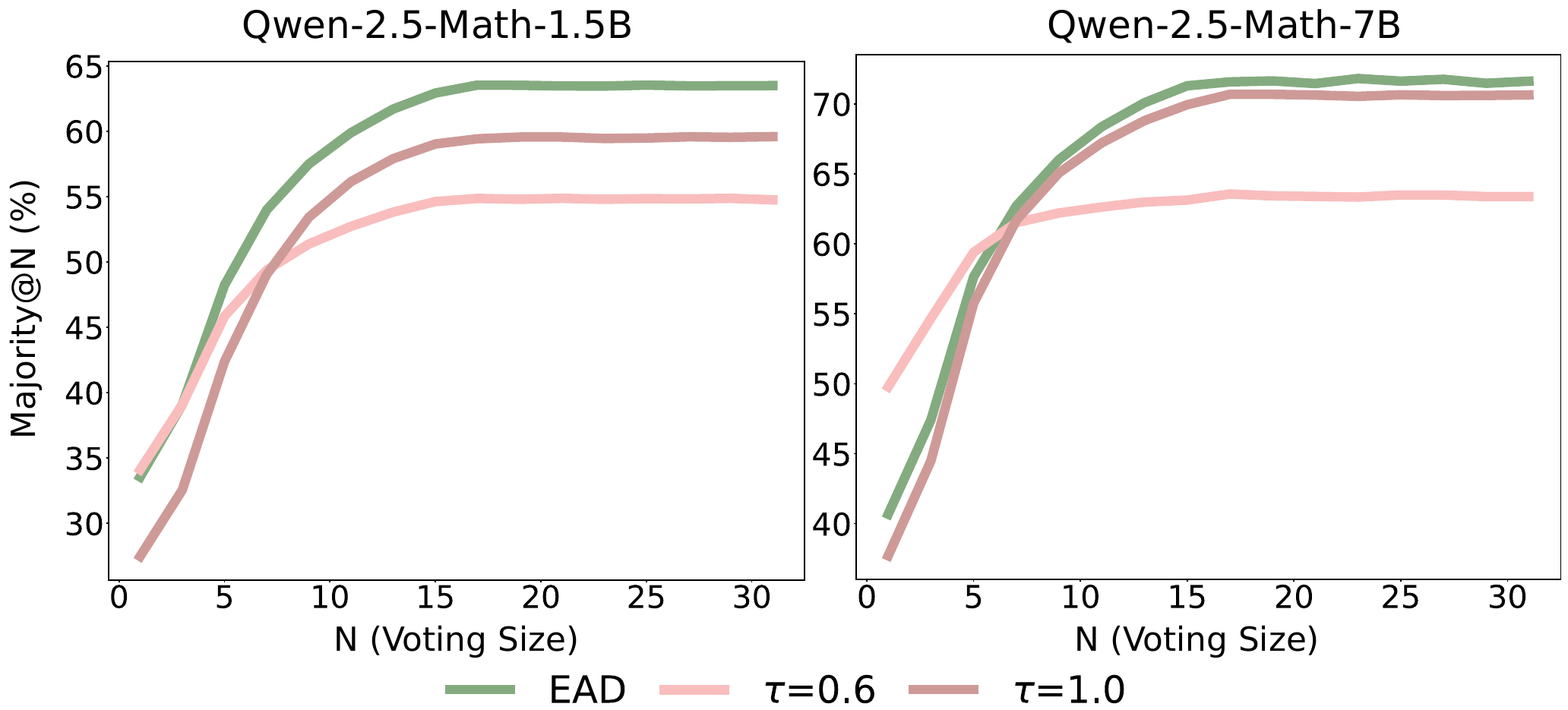}
    \vspace{-0.25in}
    \caption{Inference-Time Scaling Evaluation for Different Decoding Methods using off-the-shelf Qwen2.5 models. We could see that \alg improves traditional temperature sampling. We set $\tau_{\text{max}}=1.2, \tau_{\text{min}}=0.1, d=25$ for \alg. }
    \label{fig: inference_time_scaling}
\end{wrapfigure}
To understand whether the success of \alg in RL training is driven by its ability to generate high-quality samples, we conduct an evaluation at inference time. 
This experiment is designed to isolate the sampling strategy's effectiveness from the dynamics of RL optimization~\citep{berseth2025exploration}. 
Using off-the-shelf Qwen-2.5 models without any fine-tuning, we compare \alg against fixed-temperature sampling. 
We use majority voting ($\text{Majority}@N$) to measure how performance scales with the number of samples $N$~\citep{wang2023selfconsistency, snell2024scaling}. 
As shown in \cref{fig: inference_time_scaling}, \alg consistently improves over the baseline for most values of $N$. 
This result confirms that \alg's advantage stems from its inherent capacity to discover higher-quality solutions, even without any training.

\subsection{\alg is Compatible with Various RL Algorithms}
To demonstrate that \alg is a general, plug-and-play exploration strategy, we evaluate its performance when integrated into two other prominent RL algorithms: GRPO~\citep{shao2024deepseekmath} and EntropyMech~\citep{cui2025entropy}. 
These algorithms provide diverse testbeds. 
GRPO is more conservative, constraining policy updates with a KL divergence penalty and stricter clipping mechanism that can limit exploration~\citep{yu2025dapo}, while EntropyMech uses a specialized token-clipping mechanism to mitigate entropy collapse. 
\begin{figure}[t]
    \centering
    \includegraphics[width=.75\linewidth]{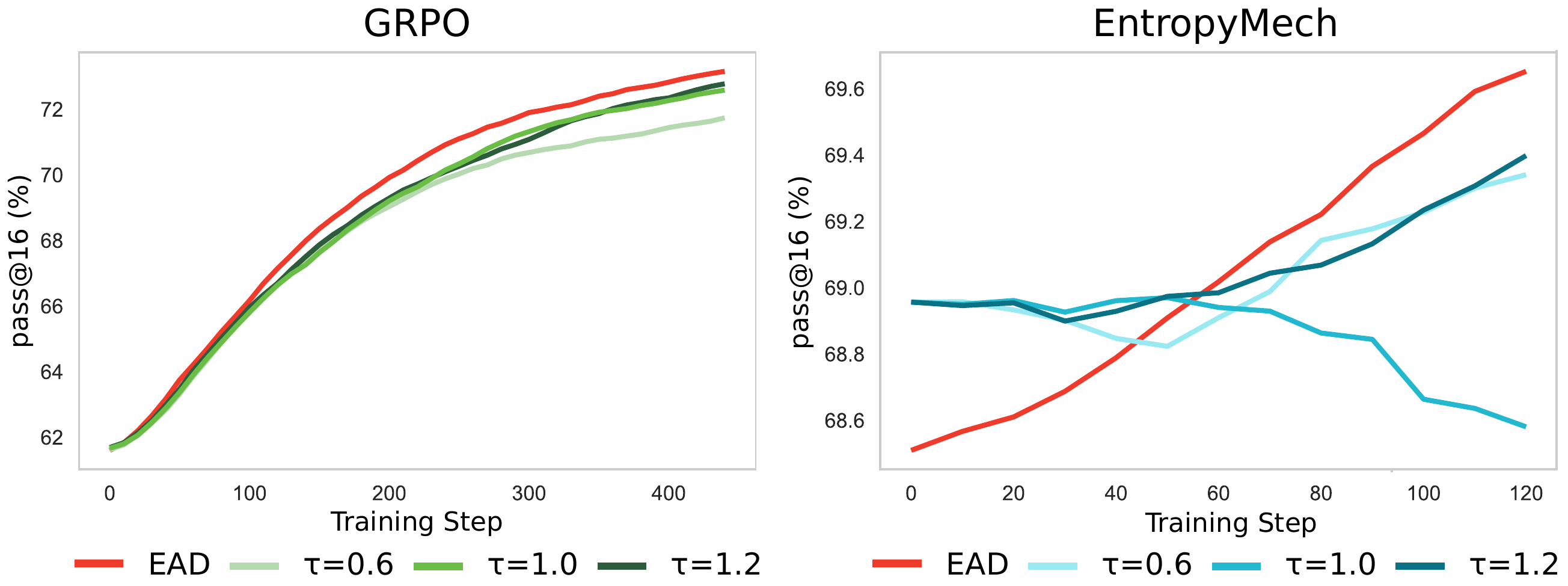}
    \vspace{-0.1in}
    \caption{\alg is compatible with various RL algorithms and can significantly improve the model performance over time. } 
    \label{fig: algorithm_wide_comparison}
    \vspace{-0.2in}
\end{figure}
As shown in \cref{fig: algorithm_wide_comparison}, \alg consistently outperforms fixed-temperature sampling in both frameworks. 
These results confirm the broad applicability of our method as an improved exploration strategy across different RL algorithms.

\section{Related Work}
\label{sec: related_work}
\paragraph{Reinforcement Learning with Verifiable Rewards.}  Recent large-scale reasoning models such as OpenAI o1~\citep{openai2024learning}, DeepSeek‑R1~\citep{guo2025deepseek} have demonstrated that reinforcement-learning-based post‑training can substantially enhance LLM reasoning. Motivated by this reinforcement learning with verifiable rewards (RLVR) (e.g.~\citet{shao2024deepseekmath,guo2025deepseek,lambert2024tulu,yang2025qwen3,hu2025open,yu2025dapo,guan2025rstar,zeng2025simplerl} among many other works) has become a major approach for post‑training LLMs to improve reasoning. A broad literature studies how to make RLVR training effective and efficient at scale, including novel reinforcement learning algorithms and objectives~\citep{yu2025dapo,liu2025understanding,yue2025vapo,zheng2025group}, verifier architecture and reward deigns~\citep{zuo2025ttrl,zhao2025learning,agarwal2025unreasonable,prabhudesai2025maximizing}, and mechanisms that manage exploration diversity and entropy~\citep{cheng2025reasoning,chen2025pass,cui2025entropy,wang2025beyond}, or takes a critical view on the current evaluation~\citep{yue2025does,zhao2025echo,hochlehnert2025sober}. Despite steady progress, a fundamental challenge is to balance exploration and exploitation along long reasoning trajectories without brittle heuristics. We adopt a simple yet effective annealed sampling schedule that front‑loads exploration and cools later steps during rollout to encourage exploration while keeping training stable.

\paragraph{Exploration Control in RLVR.} A line of works that is close to our work studies how to control exploration and sampling while doing RLVR~\citep{hou2025t1, cheng2025reasoning,chen2025pass,cui2025entropy,wang2025beyond,xiong2025minimalist,deng2025trial,xu2025not,zheng2025act,dou2025improving, li2025treepo}. In particular, \citet{cheng2025reasoning} propose per-token entropy to focus exploration at branching tokens in sampling; \citet{chen2025pass} transform per-prompt rewards to optimize pass@$k$, guiding exploration across samples; \citet{cui2025entropy} control high-covariance tokens to prevent entropy collapse and sustain exploration; \citet{wang2025beyond} update only high-entropy tokens, concentrating exploration where decisions split. Different from prior work, this paper presents the first systematic analysis of sampling temperature and introduce a purely sampling-level annealed schedule that encourages exploration and then progressively stabilize answers, thereby enabling discovery of new solutions while yielding more stable training.

\paragraph{Simulated Annealing.} 
Simulated Annealing (SA) is a probabilistic optimization technique inspired by annealing in metallurgy, designed to find the global optimum in a large search space \citep{kirkpatrick1983optimization, bertsimas1993simulated}. The core principle involves a temperature parameter that controls the probability of accepting suboptimal states. Initially, a high temperature allows the search to escape local minima by exploring broadly (exploration). As the temperature gradually decreases, the algorithm increasingly favors better states, converging towards a high-quality solution (exploitation).
This ``coarse-to-fine" search dynamic, where high temperatures establish a solution's general structure and low temperatures refine its details, strongly parallels the generative process of LLMs \citep{yang2025alignment}. SA has been adapted in various machine learning contexts to manage the exploration-exploitation trade-off, including recent applications in graph optimization~\citep{liu2021simulated}, text editing \citep{zhang2024edt}, non-autoregressive generation \citep{israel2025enabling}, and efficient Best-of-N sampling \citep{manvi2024adaptive}. However, these prior applications invariably apply a single, uniform temperature across all positions in a generated sequence. This approach fails to account for the heterogeneous roles of tokens at different positions \citep{wang2025beyond}. Our work departs from this convention. To the best of our knowledge, we are the first to introduce an \textbf{intra-sequence annealed temperature} schedule, where the temperature varies dynamically within the generation of a single sequence. This novel approach allows for more nuanced control over exploration and leads to significant performance gains in RLVR.

\section{Discussion}
Our work addresses a central challenge in RLVR: achieve an effective balance between exploration and exploitation.
We introduce \algname (\alg), a simple yet powerful sampling strategy that avoids heavy computation and intricate heuristics.
Specifically, \alg employs a temperature-annealing schedule that begins with a high sampling temperature and gradually cools, enabling LLMs to explore broadly at the beginning of generation and converge toward precise, high-quality completions throughout the decoding process.
As RLVR often relies on multiple rollouts to estimate rewards, this annealing schedule effectively improves sampling diversity while controlling variance, making it well suited for RL training.
At the same time, \alg can also be applied directly at test time to enhance inference efficiency and scaling, improving the quality of single- or multi-sample decoding without additional computation cost.

Despite the encouraging results, our study has several limitations that suggest directions for future work. 
First, the scaling behavior of our method is not fully explored because of limited computational resources. 
We adopt the current settings with reference to~\citep{xiong2025minimalist,shao2025spurious,wang2025reinforcement}, and argue that the efficacy of the method is still convincing, as we evaluate it across diverse model structures (LLaMA and Qwen) and multiple model sizes. 
A systematic scaling study remains an important next step.
Second, while \alg is designed as a complementary component, a comprehensive study combining it with other advanced exploration-promoting RLVR algorithms (see \cref{sec: related_work}) remains a promising direction for future work.
Third, our current experiments adopt a uniform temperature schedule for all prompts. 
Although an adaptive schedule tailored to individual prompts could potentially enhance performance, developing such a mechanism is nontrivial. 
In RLVR, training is iterative, so any prior information about prompt distributions may shift during optimization, and collecting extra statistics (e.g., token-wise entropy quantile~\citep{wang2025beyond}, or probability-advantage covariance~\citep{cui2025entropy}) to track these changes for every prompt would add computational overhead and system complexity~\citep{li2025treepo, liu2025uniform}.
For these reasons, we focus on the vanilla schedule to test the core efficacy of our method, leaving adaptive scheduling for future investigation.

In summary, \alg provides a simple yet general way to couple exploration with the inherent progression of language generation. 
By reducing algorithmic overhead while improving trajectory quality, it opens new avenues for both efficient inference and effective reinforcement fine-tuning.

\clearpage
\bibliography{arxiv}

\begin{thebibliography}{74}
\providecommand{\natexlab}[1]{#1}
\providecommand{\url}[1]{\texttt{#1}}
\expandafter\ifx\csname urlstyle\endcsname\relax
  \providecommand{\doi}[1]{doi: #1}\else
  \providecommand{\doi}{doi: \begingroup \urlstyle{rm}\Url}\fi

\bibitem[Ackley et~al.(1985)Ackley, Hinton, and Sejnowski]{ACKLEY1985147}
David~H. Ackley, Geoffrey~E. Hinton, and Terrence~J. Sejnowski.
\newblock A learning algorithm for boltzmann machines.
\newblock \emph{Cognitive Science}, 9\penalty0 (1):\penalty0 147--169, 1985.
\newblock ISSN 0364-0213.
\newblock \doi{https://doi.org/10.1016/S0364-0213(85)80012-4}.
\newblock URL
  \url{https://www.sciencedirect.com/science/article/pii/S0364021385800124}.

\bibitem[Agarwal et~al.(2025)Agarwal, Zhang, Yuan, Han, and
  Peng]{agarwal2025unreasonable}
Shivam Agarwal, Zimin Zhang, Lifan Yuan, Jiawei Han, and Hao Peng.
\newblock The unreasonable effectiveness of entropy minimization in llm
  reasoning.
\newblock \emph{arXiv preprint arXiv:2505.15134}, 2025.

\bibitem[Beeching et~al.(2024)Beeching, Huang, Jiang, Li, Lipkin, Qina, Rasul,
  Shen, Soletskyi, and Tunstall]{numina_math_7b}
Edward Beeching, Shengyi~Costa Huang, Albert Jiang, Jia Li, Benjamin Lipkin,
  Zihan Qina, Kashif Rasul, Ziju Shen, Roman Soletskyi, and Lewis Tunstall.
\newblock Numinamath 7b cot.
\newblock \url{https://huggingface.co/AI-MO/NuminaMath-7B-CoT}, 2024.

\bibitem[Berseth(2025)]{berseth2025exploration}
Glen Berseth.
\newblock Is exploration or optimization the problem for deep reinforcement
  learning?
\newblock \emph{arXiv preprint arXiv:2508.01329}, 2025.

\bibitem[Bertsimas \& Tsitsiklis(1993)Bertsimas and
  Tsitsiklis]{bertsimas1993simulated}
Dimitris Bertsimas and John Tsitsiklis.
\newblock Simulated annealing.
\newblock \emph{Statistical science}, 8\penalty0 (1):\penalty0 10--15, 1993.

\bibitem[Chen et~al.(2025)Chen, Qin, Wu, Ling, Ye, Zhao, and Shi]{chen2025pass}
Zhipeng Chen, Xiaobo Qin, Youbin Wu, Yue Ling, Qinghao Ye, Wayne~Xin Zhao, and
  Guang Shi.
\newblock Pass@ k training for adaptively balancing exploration and
  exploitation of large reasoning models.
\newblock \emph{arXiv preprint arXiv:2508.10751}, 2025.

\bibitem[Cheng et~al.(2025)Cheng, Huang, Zhu, Dai, Zhao, Zhang, and
  Wei]{cheng2025reasoning}
Daixuan Cheng, Shaohan Huang, Xuekai Zhu, Bo~Dai, Wayne~Xin Zhao, Zhenliang
  Zhang, and Furu Wei.
\newblock Reasoning with exploration: An entropy perspective.
\newblock \emph{arXiv preprint arXiv:2506.14758}, 2025.

\bibitem[Cui et~al.(2025)Cui, Zhang, Chen, Yuan, Wang, Zuo, Li, Fan, Chen,
  Chen, et~al.]{cui2025entropy}
Ganqu Cui, Yuchen Zhang, Jiacheng Chen, Lifan Yuan, Zhi Wang, Yuxin Zuo,
  Haozhan Li, Yuchen Fan, Huayu Chen, Weize Chen, et~al.
\newblock The entropy mechanism of reinforcement learning for reasoning
  language models.
\newblock \emph{arXiv preprint arXiv:2505.22617}, 2025.

\bibitem[Degris et~al.(2012)Degris, White, and Sutton]{degris2012off}
Thomas Degris, Martha White, and Richard~S Sutton.
\newblock Off-policy actor-critic.
\newblock \emph{arXiv preprint arXiv:1205.4839}, 2012.

\bibitem[Deng et~al.(2025)Deng, Chen, Chen, Cheng, Bai, Zhang, Min, Gao, Zhao,
  and Wen]{deng2025trial}
Jia Deng, Jie Chen, Zhipeng Chen, Daixuan Cheng, Fei Bai, Beichen Zhang,
  Yinqian Min, Yanzipeng Gao, Wayne~Xin Zhao, and Ji-Rong Wen.
\newblock From trial-and-error to improvement: A systematic analysis of llm
  exploration mechanisms in rlvr.
\newblock \emph{arXiv preprint arXiv:2508.07534}, 2025.

\bibitem[Dou et~al.(2025)Dou, Wu, Xu, Zheng, Gui, Zhang, and
  Huang]{dou2025improving}
Shihan Dou, Muling Wu, Jingwen Xu, Rui Zheng, Tao Gui, Qi~Zhang, and Xuanjing
  Huang.
\newblock Improving rl exploration for llm reasoning through retrospective
  replay.
\newblock \emph{arXiv preprint arXiv:2504.14363}, 2025.

\bibitem[Dubey et~al.(2024)Dubey, Jauhri, Pandey, Kadian, Al-Dahle, Letman,
  Mathur, Schelten, Yang, Fan, et~al.]{dubey2024llama}
Abhimanyu Dubey, Abhinav Jauhri, Abhinav Pandey, Abhishek Kadian, Ahmad
  Al-Dahle, Aiesha Letman, Akhil Mathur, Alan Schelten, Amy Yang, Angela Fan,
  et~al.
\newblock The llama 3 herd of models.
\newblock \emph{arXiv e-prints}, pp.\  arXiv--2407, 2024.

\bibitem[Fu et~al.(2025)Fu, Wang, Tian, and Zhao]{fu2025deep}
Yichao Fu, Xuewei Wang, Yuandong Tian, and Jiawei Zhao.
\newblock Deep think with confidence.
\newblock \emph{arXiv preprint arXiv:2508.15260}, 2025.

\bibitem[Gao et~al.(2023)Gao, Schulman, and Hilton]{gao2023scaling}
Leo Gao, John Schulman, and Jacob Hilton.
\newblock Scaling laws for reward model overoptimization.
\newblock In \emph{International Conference on Machine Learning}, pp.\
  10835--10866. PMLR, 2023.

\bibitem[Grattafiori et~al.(2024)Grattafiori, Dubey, Jauhri, Pandey, Kadian,
  Al-Dahle, Letman, Mathur, Schelten, Vaughan, et~al.]{grattafiori2024llama}
Aaron Grattafiori, Abhimanyu Dubey, Abhinav Jauhri, Abhinav Pandey, Abhishek
  Kadian, Ahmad Al-Dahle, Aiesha Letman, Akhil Mathur, Alan Schelten, Alex
  Vaughan, et~al.
\newblock The llama 3 herd of models.
\newblock \emph{arXiv preprint arXiv:2407.21783}, 2024.

\bibitem[Guan et~al.(2025)Guan, Zhang, Liu, Shang, Sun, Zhu, Yang, and
  Yang]{guan2025rstar}
Xinyu Guan, Li~Lyna Zhang, Yifei Liu, Ning Shang, Youran Sun, Yi~Zhu, Fan Yang,
  and Mao Yang.
\newblock rstar-math: Small llms can master math reasoning with self-evolved
  deep thinking.
\newblock \emph{arXiv preprint arXiv:2501.04519}, 2025.

\bibitem[Guo et~al.(2025)Guo, Yang, Zhang, Song, Zhang, Xu, Zhu, Ma, Wang, Bi,
  et~al.]{guo2025deepseek}
Daya Guo, Dejian Yang, Haowei Zhang, Junxiao Song, Ruoyu Zhang, Runxin Xu,
  Qihao Zhu, Shirong Ma, Peiyi Wang, Xiao Bi, et~al.
\newblock Deepseek-r1: Incentivizing reasoning capability in llms via
  reinforcement learning.
\newblock \emph{arXiv preprint arXiv:2501.12948}, 2025.

\bibitem[Haarnoja et~al.(2018)Haarnoja, Zhou, Abbeel, and
  Levine]{haarnoja2018soft}
Tuomas Haarnoja, Aurick Zhou, Pieter Abbeel, and Sergey Levine.
\newblock Soft actor-critic: Off-policy maximum entropy deep reinforcement
  learning with a stochastic actor.
\newblock In \emph{International conference on machine learning}, pp.\
  1861--1870. Pmlr, 2018.

\bibitem[Heckman et~al.(1998)Heckman, Ichimura, and Todd]{heckman1998matching}
James~J Heckman, Hidehiko Ichimura, and Petra Todd.
\newblock Matching as an econometric evaluation estimator.
\newblock \emph{The review of economic studies}, 65\penalty0 (2):\penalty0
  261--294, 1998.

\bibitem[Hendrycks et~al.(2021)Hendrycks, Burns, Basart, Zou, Mazeika, Song,
  and Steinhardt]{hendrycks2021measuring}
Dan Hendrycks, Collin Burns, Steven Basart, Andy Zou, Mantas Mazeika, Dawn
  Song, and Jacob Steinhardt.
\newblock Measuring massive multitask language understanding.
\newblock In \emph{International Conference on Learning Representations}, 2021.
\newblock URL \url{https://openreview.net/forum?id=d7KBjmI3GmQ}.

\bibitem[Hilton et~al.(2022)Hilton, Cobbe, and Schulman]{hilton2022batch}
Jacob Hilton, Karl Cobbe, and John Schulman.
\newblock Batch size-invariance for policy optimization.
\newblock In S.~Koyejo, S.~Mohamed, A.~Agarwal, D.~Belgrave, K.~Cho, and A.~Oh
  (eds.), \emph{Advances in Neural Information Processing Systems}, volume~35,
  pp.\  17086--17098. Curran Associates, Inc., 2022.
\newblock URL
  \url{https://proceedings.neurips.cc/paper_files/paper/2022/file/6ceb6c2150bbf46fd75528a6cd6be793-Paper-Conference.pdf}.

\bibitem[Hochlehnert et~al.(2025)Hochlehnert, Bhatnagar, Udandarao, Albanie,
  Prabhu, and Bethge]{hochlehnert2025sober}
Andreas Hochlehnert, Hardik Bhatnagar, Vishaal Udandarao, Samuel Albanie, Ameya
  Prabhu, and Matthias Bethge.
\newblock A sober look at progress in language model reasoning: Pitfalls and
  paths to reproducibility.
\newblock \emph{arXiv preprint arXiv:2504.07086}, 2025.

\bibitem[Holtzman et~al.(2020)Holtzman, Buys, Du, Forbes, and
  Choi]{Holtzman2020The}
Ari Holtzman, Jan Buys, Li~Du, Maxwell Forbes, and Yejin Choi.
\newblock The curious case of neural text degeneration.
\newblock In \emph{International Conference on Learning Representations}, 2020.
\newblock URL \url{https://openreview.net/forum?id=rygGQyrFvH}.

\bibitem[Hou et~al.(2025)Hou, Lv, Lu, Zhang, Li, Yao, Li, Tang, and
  Dong]{hou2025t1}
Zhenyu Hou, Xin Lv, Rui Lu, Jiajie Zhang, Yujiang Li, Zijun Yao, Juanzi Li, Jie
  Tang, and Yuxiao Dong.
\newblock T1: Advancing language model reasoning through reinforcement learning
  and inference scaling.
\newblock In \emph{Forty-second International Conference on Machine Learning},
  2025.

\bibitem[Hu et~al.(2025)Hu, Zhang, Han, Jiang, Zhang, and Shum]{hu2025open}
Jingcheng Hu, Yinmin Zhang, Qi~Han, Daxin Jiang, Xiangyu Zhang, and Heung-Yeung
  Shum.
\newblock Open-reasoner-zero: An open source approach to scaling up
  reinforcement learning on the base model.
\newblock \emph{arXiv preprint arXiv:2503.24290}, 2025.

\bibitem[Israel et~al.(2025)Israel, Grover, and Broeck]{israel2025enabling}
Daniel Israel, Aditya Grover, and Guy Van~den Broeck.
\newblock Enabling autoregressive models to fill in masked tokens.
\newblock \emph{arXiv preprint arXiv:2502.06901}, 2025.

\bibitem[Jaynes(1957)]{jaynes1957information}
Edwin~T Jaynes.
\newblock Information theory and statistical mechanics.
\newblock \emph{Physical review}, 106\penalty0 (4):\penalty0 620, 1957.

\bibitem[Kirkpatrick et~al.(1983)Kirkpatrick, Gelatt~Jr, and
  Vecchi]{kirkpatrick1983optimization}
Scott Kirkpatrick, C~Daniel Gelatt~Jr, and Mario~P Vecchi.
\newblock Optimization by simulated annealing.
\newblock \emph{science}, 220\penalty0 (4598):\penalty0 671--680, 1983.

\bibitem[Ladosz et~al.(2022)Ladosz, Weng, Kim, and Oh]{ladosz2022exploration}
Pawel Ladosz, Lilian Weng, Minwoo Kim, and Hyondong Oh.
\newblock Exploration in deep reinforcement learning: A survey.
\newblock \emph{Information Fusion}, 85:\penalty0 1--22, 2022.

\bibitem[Lambert et~al.(2024)Lambert, Morrison, Pyatkin, Huang, Ivison,
  Brahman, Miranda, Liu, Dziri, Lyu, et~al.]{lambert2024tulu}
Nathan Lambert, Jacob Morrison, Valentina Pyatkin, Shengyi Huang, Hamish
  Ivison, Faeze Brahman, Lester James~V Miranda, Alisa Liu, Nouha Dziri, Shane
  Lyu, et~al.
\newblock Tulu 3: Pushing frontiers in open language model post-training.
\newblock \emph{arXiv preprint arXiv:2411.15124}, 2024.

\bibitem[Li et~al.(2025)Li, Gu, Wen, Li, Xing, Guo, Zheng, Zhou, Qu, Zhou,
  et~al.]{li2025treepo}
Yizhi Li, Qingshui Gu, Zhoufutu Wen, Ziniu Li, Tianshun Xing, Shuyue Guo,
  Tianyu Zheng, Xin Zhou, Xingwei Qu, Wangchunshu Zhou, et~al.
\newblock Treepo: Bridging the gap of policy optimization and efficacy and
  inference efficiency with heuristic tree-based modeling.
\newblock \emph{arXiv preprint arXiv:2508.17445}, 2025.

\bibitem[Liao et~al.(2025)Liao, Chen, Rajaee, Xu, Herold, S{\o}gaard, de~Rijke,
  and Monz]{liao2025lost}
Baohao Liao, Xinyi Chen, Sara Rajaee, Yuhui Xu, Christian Herold, Anders
  S{\o}gaard, Maarten de~Rijke, and Christof Monz.
\newblock Lost at the beginning of reasoning.
\newblock \emph{arXiv preprint arXiv:2506.22058}, 2025.

\bibitem[Liu et~al.(2021)Liu, Li, Meng, Zhou, Zhong, Zhou, Mou, and
  Song]{liu2021simulated}
Xianggen Liu, Pengyong Li, Fandong Meng, Hao Zhou, Huasong Zhong, Jie Zhou,
  Lili Mou, and Sen Song.
\newblock Simulated annealing for optimization of graphs and sequences.
\newblock \emph{Neurocomputing}, 465:\penalty0 310--324, 2021.

\bibitem[Liu et~al.(2025{\natexlab{a}})Liu, Liu, Wen, Cai, Cui, He, and
  Zhang]{liu2025uniform}
Zheng Liu, Mengjie Liu, Siwei Wen, Mengzhang Cai, Bin Cui, Conghui He, and
  Wentao Zhang.
\newblock From uniform to heterogeneous: Tailoring policy optimization to every
  token's nature.
\newblock \emph{arXiv preprint arXiv:2509.16591}, 2025{\natexlab{a}}.

\bibitem[Liu et~al.(2025{\natexlab{b}})Liu, Chen, Li, Qi, Pang, Du, Lee, and
  Lin]{liu2025understanding}
Zichen Liu, Changyu Chen, Wenjun Li, Penghui Qi, Tianyu Pang, Chao Du, Wee~Sun
  Lee, and Min Lin.
\newblock Understanding r1-zero-like training: A critical perspective.
\newblock \emph{arXiv preprint arXiv:2503.20783}, 2025{\natexlab{b}}.

\bibitem[Manvi et~al.(2024)Manvi, Singh, and Ermon]{manvi2024adaptive}
Rohin Manvi, Anikait Singh, and Stefano Ermon.
\newblock Adaptive inference-time compute: Llms can predict if they can do
  better, even mid-generation.
\newblock \emph{arXiv preprint arXiv:2410.02725}, 2024.

\bibitem[{OpenAI}(2024)]{openai2024learning}
{OpenAI}.
\newblock Learning to reason with llms.
\newblock \url{https://openai.com/index/learning-to-reason-with-llms/}, 2024.
\newblock Accessed: 2025-05-01.

\bibitem[Perez et~al.(2023)Perez, Ringer, Lukosiute, Nguyen, Chen, Heiner,
  Pettit, Olsson, Kundu, Kadavath, et~al.]{perez2023discovering}
Ethan Perez, Sam Ringer, Kamile Lukosiute, Karina Nguyen, Edwin Chen, Scott
  Heiner, Craig Pettit, Catherine Olsson, Sandipan Kundu, Saurav Kadavath,
  et~al.
\newblock Discovering language model behaviors with model-written evaluations.
\newblock In \emph{Findings of the association for computational linguistics:
  ACL 2023}, pp.\  13387--13434, 2023.

\bibitem[Prabhudesai et~al.(2025)Prabhudesai, Chen, Ippoliti, Fragkiadaki, Liu,
  and Pathak]{prabhudesai2025maximizing}
Mihir Prabhudesai, Lili Chen, Alex Ippoliti, Katerina Fragkiadaki, Hao Liu, and
  Deepak Pathak.
\newblock Maximizing confidence alone improves reasoning.
\newblock \emph{arXiv preprint arXiv:2505.22660}, 2025.

\bibitem[Renze(2024)]{renze-2024-effect}
Matthew Renze.
\newblock The effect of sampling temperature on problem solving in large
  language models.
\newblock In Yaser Al-Onaizan, Mohit Bansal, and Yun-Nung Chen (eds.),
  \emph{Findings of the Association for Computational Linguistics: EMNLP 2024},
  pp.\  7346--7356, Miami, Florida, USA, November 2024. Association for
  Computational Linguistics.
\newblock \doi{10.18653/v1/2024.findings-emnlp.432}.
\newblock URL \url{https://aclanthology.org/2024.findings-emnlp.432/}.

\bibitem[Schulman et~al.(2017)Schulman, Wolski, Dhariwal, Radford, and
  Klimov]{schulman2017proximal}
John Schulman, Filip Wolski, Prafulla Dhariwal, Alec Radford, and Oleg Klimov.
\newblock Proximal policy optimization algorithms.
\newblock \emph{arXiv preprint arXiv:1707.06347}, 2017.

\bibitem[Shannon(1948)]{shannon1948mathematical}
Claude~E Shannon.
\newblock A mathematical theory of communication.
\newblock \emph{The Bell system technical journal}, 27\penalty0 (3):\penalty0
  379--423, 1948.

\bibitem[Shao et~al.(2025)Shao, Li, Xin, Geng, Wang, Oh, Du, Lambert, Min,
  Krishna, et~al.]{shao2025spurious}
Rulin Shao, Shuyue~Stella Li, Rui Xin, Scott Geng, Yiping Wang, Sewoong Oh,
  Simon~Shaolei Du, Nathan Lambert, Sewon Min, Ranjay Krishna, et~al.
\newblock Spurious rewards: Rethinking training signals in rlvr.
\newblock \emph{arXiv preprint arXiv:2506.10947}, 2025.

\bibitem[Shao et~al.(2024)Shao, Wang, Zhu, Xu, Song, Bi, Zhang, Zhang, Li, Wu,
  et~al.]{shao2024deepseekmath}
Zhihong Shao, Peiyi Wang, Qihao Zhu, Runxin Xu, Junxiao Song, Xiao Bi, Haowei
  Zhang, Mingchuan Zhang, YK~Li, Yang Wu, et~al.
\newblock Deepseekmath: Pushing the limits of mathematical reasoning in open
  language models.
\newblock \emph{arXiv preprint arXiv:2402.03300}, 2024.

\bibitem[Shrivastava et~al.(2025)Shrivastava, Awadallah, Balachandran, Garg,
  Behl, and Papailiopoulos]{shrivastava2025sample}
Vaishnavi Shrivastava, Ahmed Awadallah, Vidhisha Balachandran, Shivam Garg,
  Harkirat Behl, and Dimitris Papailiopoulos.
\newblock Sample more to think less: Group filtered policy optimization for
  concise reasoning.
\newblock \emph{arXiv preprint arXiv:2508.09726}, 2025.

\bibitem[Snell et~al.(2024)Snell, Lee, Xu, and Kumar]{snell2024scaling}
Charlie Snell, Jaehoon Lee, Kelvin Xu, and Aviral Kumar.
\newblock Scaling llm test-time compute optimally can be more effective than
  scaling model parameters.
\newblock \emph{arXiv preprint arXiv:2408.03314}, 2024.

\bibitem[Spataru et~al.(2024)Spataru, Hambro, Voita, and
  Cancedda]{spataru2024know}
Ava Spataru, Eric Hambro, Elena Voita, and Nicola Cancedda.
\newblock Know when to stop: A study of semantic drift in text generation.
\newblock \emph{arXiv preprint arXiv:2404.05411}, 2024.

\bibitem[Sutton et~al.(1998)Sutton, Barto, et~al.]{sutton1998reinforcement}
Richard~S Sutton, Andrew~G Barto, et~al.
\newblock \emph{Reinforcement learning: An introduction}, volume~1.
\newblock MIT press Cambridge, 1998.

\bibitem[Team et~al.(2025)Team, Du, Gao, Xing, Jiang, Chen, Li, Xiao, Du, Liao,
  et~al.]{team2025kimi}
Kimi Team, Angang Du, Bofei Gao, Bowei Xing, Changjiu Jiang, Cheng Chen, Cheng
  Li, Chenjun Xiao, Chenzhuang Du, Chonghua Liao, et~al.
\newblock Kimi k1. 5: Scaling reinforcement learning with llms.
\newblock \emph{arXiv preprint arXiv:2501.12599}, 2025.

\bibitem[Thrun(1992)]{thrun1992efficient}
Sebastian~B Thrun.
\newblock \emph{Efficient exploration in reinforcement learning}.
\newblock Carnegie Mellon University, 1992.

\bibitem[Wang et~al.(2025{\natexlab{a}})Wang, Zhao, Jiang, Chen, Zhu, Chen,
  Liu, Zhang, Fan, Ma, et~al.]{wcq2025beyond}
Chaoqi Wang, Zhuokai Zhao, Yibo Jiang, Zhaorun Chen, Chen Zhu, Yuxin Chen,
  Jiayi Liu, Lizhu Zhang, Xiangjun Fan, Hao Ma, et~al.
\newblock Beyond reward hacking: Causal rewards for large language model
  alignment.
\newblock \emph{arXiv preprint arXiv:2501.09620}, 2025{\natexlab{a}}.

\bibitem[Wang et~al.(2025{\natexlab{b}})Wang, Yu, Gao, Zheng, Liu, Lu, Dang,
  Chen, Yang, Zhang, et~al.]{wang2025beyond}
Shenzhi Wang, Le~Yu, Chang Gao, Chujie Zheng, Shixuan Liu, Rui Lu, Kai Dang,
  Xionghui Chen, Jianxin Yang, Zhenru Zhang, et~al.
\newblock Beyond the 80/20 rule: High-entropy minority tokens drive effective
  reinforcement learning for llm reasoning.
\newblock \emph{arXiv preprint arXiv:2506.01939}, 2025{\natexlab{b}}.

\bibitem[Wang et~al.(2023)Wang, Wei, Schuurmans, Le, Chi, Narang, Chowdhery,
  and Zhou]{wang2023selfconsistency}
Xuezhi Wang, Jason Wei, Dale Schuurmans, Quoc~V Le, Ed~H. Chi, Sharan Narang,
  Aakanksha Chowdhery, and Denny Zhou.
\newblock Self-consistency improves chain of thought reasoning in language
  models.
\newblock In \emph{The Eleventh International Conference on Learning
  Representations}, 2023.
\newblock URL \url{https://openreview.net/forum?id=1PL1NIMMrw}.

\bibitem[Wang et~al.(2025{\natexlab{c}})Wang, Yang, Zeng, Ren, Liu, Peng,
  Cheng, He, Wang, Gao, et~al.]{wang2025reinforcement}
Yiping Wang, Qing Yang, Zhiyuan Zeng, Liliang Ren, Liyuan Liu, Baolin Peng, Hao
  Cheng, Xuehai He, Kuan Wang, Jianfeng Gao, et~al.
\newblock Reinforcement learning for reasoning in large language models with
  one training example.
\newblock \emph{arXiv preprint arXiv:2504.20571}, 2025{\natexlab{c}}.

\bibitem[Wang et~al.(2025{\natexlab{d}})Wang, Zhou, Li, and
  Liu]{wang2025octothinker}
Zengzhi Wang, Fan Zhou, Xuefeng Li, and Pengfei Liu.
\newblock Octothinker: Mid-training incentivizes reinforcement learning
  scaling.
\newblock \emph{arXiv preprint arXiv:2506.20512}, 2025{\natexlab{d}}.

\bibitem[Wei et~al.(2022)Wei, Wang, Schuurmans, Bosma, Xia, Chi, Le, Zhou,
  et~al.]{wei2022chain}
Jason Wei, Xuezhi Wang, Dale Schuurmans, Maarten Bosma, Fei Xia, Ed~Chi, Quoc~V
  Le, Denny Zhou, et~al.
\newblock Chain-of-thought prompting elicits reasoning in large language
  models.
\newblock \emph{Advances in neural information processing systems},
  35:\penalty0 24824--24837, 2022.

\bibitem[Weng(2024)]{weng2024rewardhack}
Lilian Weng.
\newblock Reward hacking in reinforcement learning.
\newblock \emph{lilianweng.github.io}, Nov 2024.
\newblock URL
  \url{https://lilianweng.github.io/posts/2024-11-28-reward-hacking/}.

\bibitem[Xiong et~al.(2025)Xiong, Yao, Xu, Pang, Wang, Sahoo, Li, Jiang, Zhang,
  Xiong, et~al.]{xiong2025minimalist}
Wei Xiong, Jiarui Yao, Yuhui Xu, Bo~Pang, Lei Wang, Doyen Sahoo, Junnan Li, Nan
  Jiang, Tong Zhang, Caiming Xiong, et~al.
\newblock A minimalist approach to llm reasoning: from rejection sampling to
  reinforce.
\newblock \emph{arXiv preprint arXiv:2504.11343}, 2025.

\bibitem[Xu et~al.(2025)Xu, Savani, Fang, and Kolter]{xu2025not}
Yixuan~Even Xu, Yash Savani, Fei Fang, and Zico Kolter.
\newblock Not all rollouts are useful: Down-sampling rollouts in llm
  reinforcement learning.
\newblock \emph{arXiv preprint arXiv:2504.13818}, 2025.

\bibitem[Yang et~al.(2024)Yang, Yang, Hui, Zheng, Yu, Zhou, Li, Li, Liu, Huang,
  et~al.]{yang2024qwen2}
An~Yang, Baosong Yang, Binyuan Hui, Bo~Zheng, Bowen Yu, Chang Zhou, Chengpeng
  Li, Chengyuan Li, Dayiheng Liu, Fei Huang, et~al.
\newblock Qwen2 technical report.
\newblock \emph{arXiv preprint arXiv:2407.10671}, 2024.

\bibitem[Yang et~al.(2025)Yang, Li, Yang, Zhang, Hui, Zheng, Yu, Gao, Huang,
  Lv, et~al.]{yang2025qwen3}
An~Yang, Anfeng Li, Baosong Yang, Beichen Zhang, Binyuan Hui, Bo~Zheng, Bowen
  Yu, Chang Gao, Chengen Huang, Chenxu Lv, et~al.
\newblock Qwen3 technical report.
\newblock \emph{arXiv preprint arXiv:2505.09388}, 2025.

\bibitem[Yang \& Holtzman(2025)Yang and Holtzman]{yang2025alignment}
Chenghao Yang and Ari Holtzman.
\newblock How alignment shrinks the generative horizon.
\newblock \emph{arXiv preprint arXiv:2506.17871}, 2025.

\bibitem[Yao et~al.(2025)Yao, Liu, Zhang, Dong, Shang, and
  Gao]{yao2025offpolicy}
Feng Yao, Liyuan Liu, Dinghuai Zhang, Chengyu Dong, Jingbo Shang, and Jianfeng
  Gao.
\newblock Your efficient rl framework secretly brings you off-policy rl
  training, August 2025.
\newblock URL \url{https://fengyao.notion.site/off-policy-rl}.

\bibitem[Yu et~al.(2025)Yu, Zhang, Zhu, Yuan, Zuo, Yue, Dai, Fan, Liu, Liu,
  et~al.]{yu2025dapo}
Qiying Yu, Zheng Zhang, Ruofei Zhu, Yufeng Yuan, Xiaochen Zuo, Yu~Yue, Weinan
  Dai, Tiantian Fan, Gaohong Liu, Lingjun Liu, et~al.
\newblock Dapo: An open-source llm reinforcement learning system at scale.
\newblock \emph{arXiv preprint arXiv:2503.14476}, 2025.

\bibitem[Yue et~al.(2025{\natexlab{a}})Yue, Chen, Lu, Zhao, Wang, Song, and
  Huang]{yue2025does}
Yang Yue, Zhiqi Chen, Rui Lu, Andrew Zhao, Zhaokai Wang, Shiji Song, and Gao
  Huang.
\newblock Does reinforcement learning really incentivize reasoning capacity in
  llms beyond the base model?
\newblock \emph{arXiv preprint arXiv:2504.13837}, 2025{\natexlab{a}}.

\bibitem[Yue et~al.(2025{\natexlab{b}})Yue, Yuan, Yu, Zuo, Zhu, Xu, Chen, Wang,
  Fan, Du, et~al.]{yue2025vapo}
Yu~Yue, Yufeng Yuan, Qiying Yu, Xiaochen Zuo, Ruofei Zhu, Wenyuan Xu, Jiaze
  Chen, Chengyi Wang, TianTian Fan, Zhengyin Du, et~al.
\newblock Vapo: Efficient and reliable reinforcement learning for advanced
  reasoning tasks.
\newblock \emph{arXiv preprint arXiv:2504.05118}, 2025{\natexlab{b}}.

\bibitem[Zeng et~al.(2025)Zeng, Huang, Liu, Liu, He, Ma, and
  He]{zeng2025simplerl}
Weihao Zeng, Yuzhen Huang, Qian Liu, Wei Liu, Keqing He, Zejun Ma, and Junxian
  He.
\newblock Simplerl-zoo: Investigating and taming zero reinforcement learning
  for open base models in the wild.
\newblock \emph{arXiv preprint arXiv:2503.18892}, 2025.

\bibitem[Zhang et~al.(2024)Zhang, Bao, and Huang]{zhang2024edt}
Shimao Zhang, Yu~Bao, and Shujian Huang.
\newblock Edt: Improving large language models' generation by entropy-based
  dynamic temperature sampling.
\newblock \emph{arXiv preprint arXiv:2403.14541}, 2024.

\bibitem[Zhao et~al.(2025{\natexlab{a}})Zhao, Meterez, Kakade, Pehlevan,
  Jelassi, and Malach]{zhao2025echo}
Rosie Zhao, Alexandru Meterez, Sham Kakade, Cengiz Pehlevan, Samy Jelassi, and
  Eran Malach.
\newblock Echo chamber: Rl post-training amplifies behaviors learned in
  pretraining.
\newblock \emph{arXiv preprint arXiv:2504.07912}, 2025{\natexlab{a}}.

\bibitem[Zhao et~al.(2025{\natexlab{b}})Zhao, Kang, Feng, Levine, and
  Song]{zhao2025learning}
Xuandong Zhao, Zhewei Kang, Aosong Feng, Sergey Levine, and Dawn Song.
\newblock Learning to reason without external rewards.
\newblock \emph{arXiv preprint arXiv:2505.19590}, 2025{\natexlab{b}}.

\bibitem[Zheng et~al.(2025{\natexlab{a}})Zheng, Liu, Li, Chen, Yu, Gao, Dang,
  Liu, Men, Yang, et~al.]{zheng2025group}
Chujie Zheng, Shixuan Liu, Mingze Li, Xiong-Hui Chen, Bowen Yu, Chang Gao, Kai
  Dang, Yuqiong Liu, Rui Men, An~Yang, et~al.
\newblock Group sequence policy optimization.
\newblock \emph{arXiv preprint arXiv:2507.18071}, 2025{\natexlab{a}}.

\bibitem[Zheng et~al.(2025{\natexlab{b}})Zheng, Zhou, Bartoldson, Kailkhura,
  Lai, Zhao, and Chen]{zheng2025act}
Haizhong Zheng, Yang Zhou, Brian~R Bartoldson, Bhavya Kailkhura, Fan Lai,
  Jiawei Zhao, and Beidi Chen.
\newblock Act only when it pays: Efficient reinforcement learning for llm
  reasoning via selective rollouts.
\newblock \emph{arXiv preprint arXiv:2506.02177}, 2025{\natexlab{b}}.

\bibitem[Ziegler et~al.(2019)Ziegler, Stiennon, Wu, Brown, Radford, Amodei,
  Christiano, and Irving]{ziegler2019fine}
Daniel~M Ziegler, Nisan Stiennon, Jeffrey Wu, Tom~B Brown, Alec Radford, Dario
  Amodei, Paul Christiano, and Geoffrey Irving.
\newblock Fine-tuning language models from human preferences.
\newblock \emph{arXiv preprint arXiv:1909.08593}, 2019.

\bibitem[Zuo et~al.(2025)Zuo, Zhang, Sheng, Qu, Cui, Zhu, Li, Zhang, Long, Hua,
  et~al.]{zuo2025ttrl}
Yuxin Zuo, Kaiyan Zhang, Li~Sheng, Shang Qu, Ganqu Cui, Xuekai Zhu, Haozhan Li,
  Yuchen Zhang, Xinwei Long, Ermo Hua, et~al.
\newblock Ttrl: Test-time reinforcement learning.
\newblock \emph{arXiv preprint arXiv:2504.16084}, 2025.

\end{thebibliography}
\bibliographystyle{iclr2026_conference}

\newpage
\appendix
\section{Minimal-RL Training Details}
\label{app: minimal_rl_training_details}
We mainly follow the Minimal-RL recipe~\citep{xiong2025minimalist} in our experiments to ensure a fair comparison among different rollout sampling strategies. Specifically, we set a series of hyperparameters as in \cref{tab: hyperparameter}:

\begin{table}[h!]
\centering
\begin{tabular}{@{}cc@{}}
\toprule
Hyperparameter                                 & Value(s)        \\ 
\midrule
\multicolumn{1}{c}{Training Batch Size}      & \multicolumn{1}{c}{1024} \\ 
\multicolumn{1}{c}{Max Prompt Length}        & \multicolumn{1}{c}{1024} \\ 
\multicolumn{1}{c}{Max Response Length}      & \multicolumn{1}{c}{3072} \\ 
\multicolumn{1}{c}{Mini Batch Size}          & \multicolumn{1}{c}{256}  \\ 
\multicolumn{1}{c}{Micro Batch Size Per GPU} & \multicolumn{1}{c}{4}    \\ 
\multicolumn{1}{c}{Learning Rate}            & \multicolumn{1}{c}{$10^{-6}$} \\ 
\bottomrule
\end{tabular}%
\caption{Hyperparameter Setup for Running Minimal-RL recipe. }
\label{tab: hyperparameter}
\end{table}

\section{Off-policy Issue and Truncated Importance Sampling Correction}
\label{sec:off-policy}
\subsection{Sampling Techniques Can Introduce Off-Policy Issue}
\label{subsec:temperature-sampling}
One subtle yet troublesome drawback of reinforcement learning with sampling techniques is that it simultaneously introduces the \emph{off-policy} problem: there is a gap between the behavior policy (used for sampling) and the target policy (being optimized and parametrized by $\theta$). This might introduce instability to the training and cause it to fail (See example training of RLVR with EAD \cref{fig:optim_failure_for_as_offpolicy}).

\begin{figure}[h!]
    \centering
    \begin{subfigure}[b]{0.26\textwidth}
    \centering
    \includegraphics[width=\linewidth]{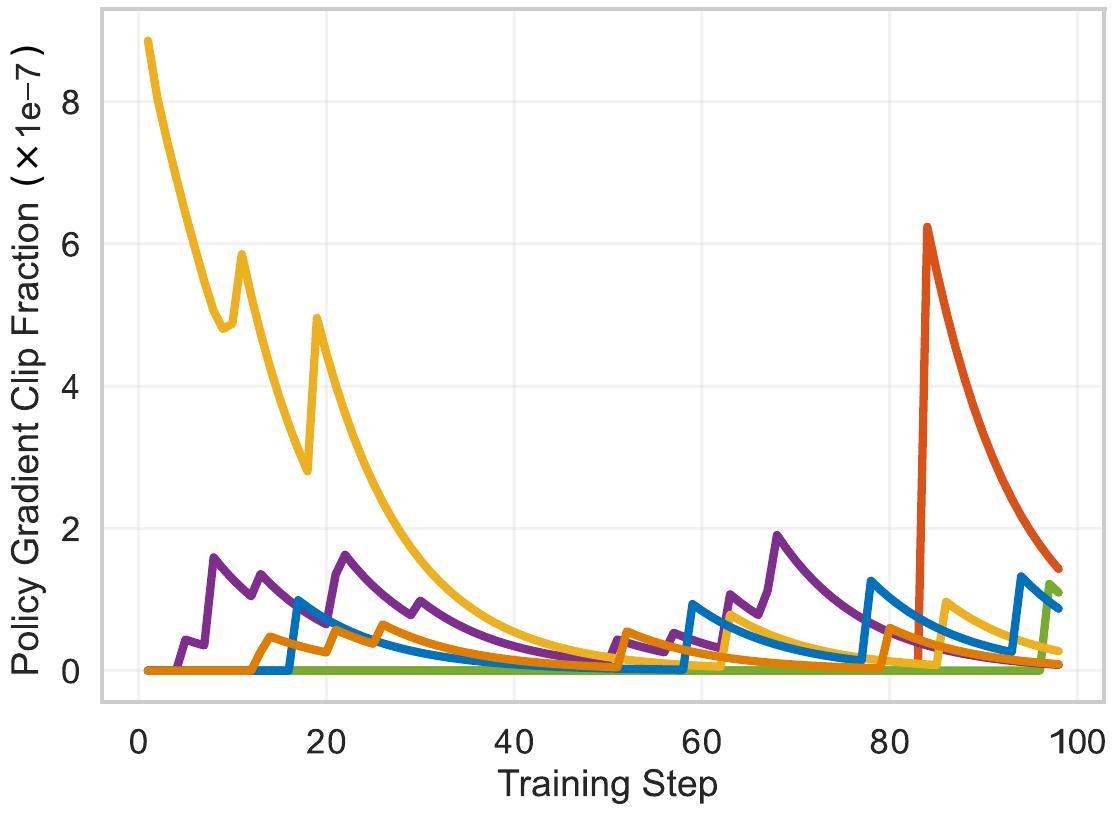}
    \caption{Clip Fraction Surge}
    \label{fig: off_policy_pg_clipfrac}
  \end{subfigure}
  \begin{subfigure}[b]{0.28\textwidth}
    \centering
    \includegraphics[width=\linewidth]{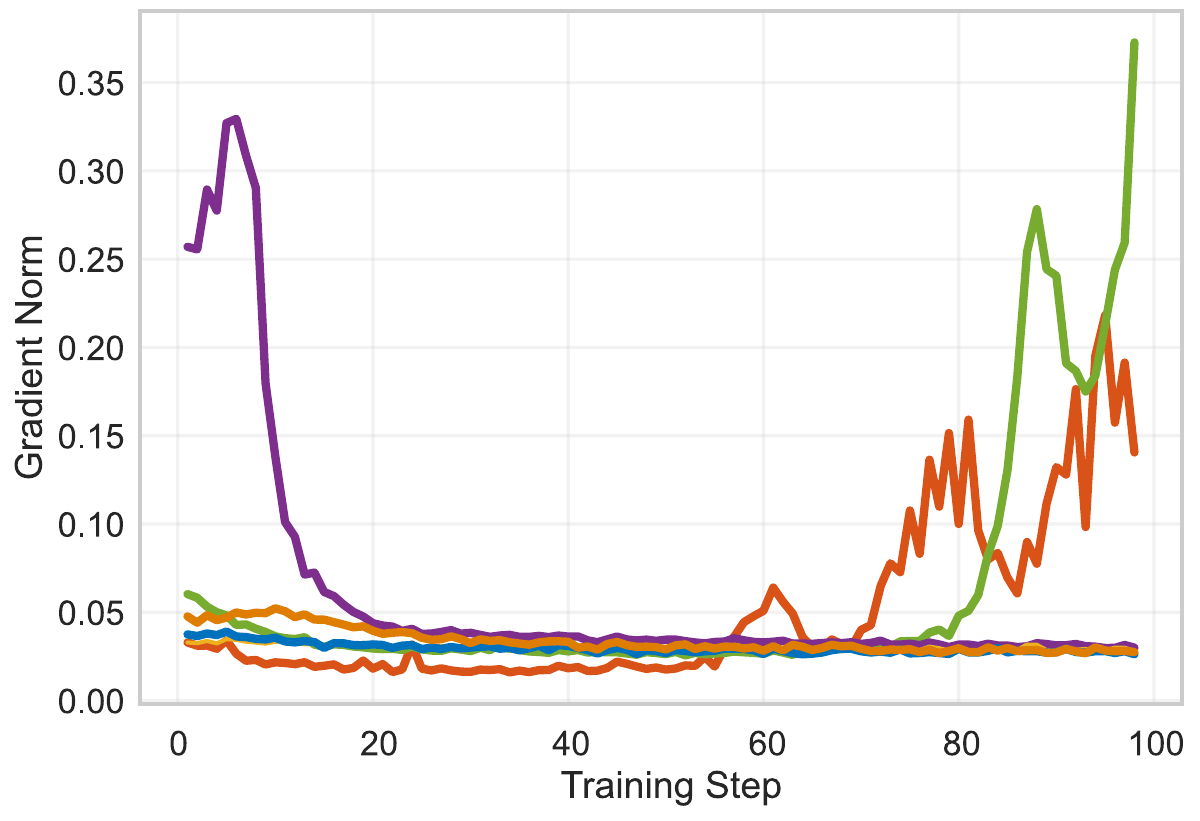}
    \caption{Gradient Norm Surge.}
    \label{fig: off_policy_grad_norm}
  \end{subfigure}
   \begin{subfigure}[b]{0.27\textwidth}
    \centering
    \includegraphics[width=\linewidth]{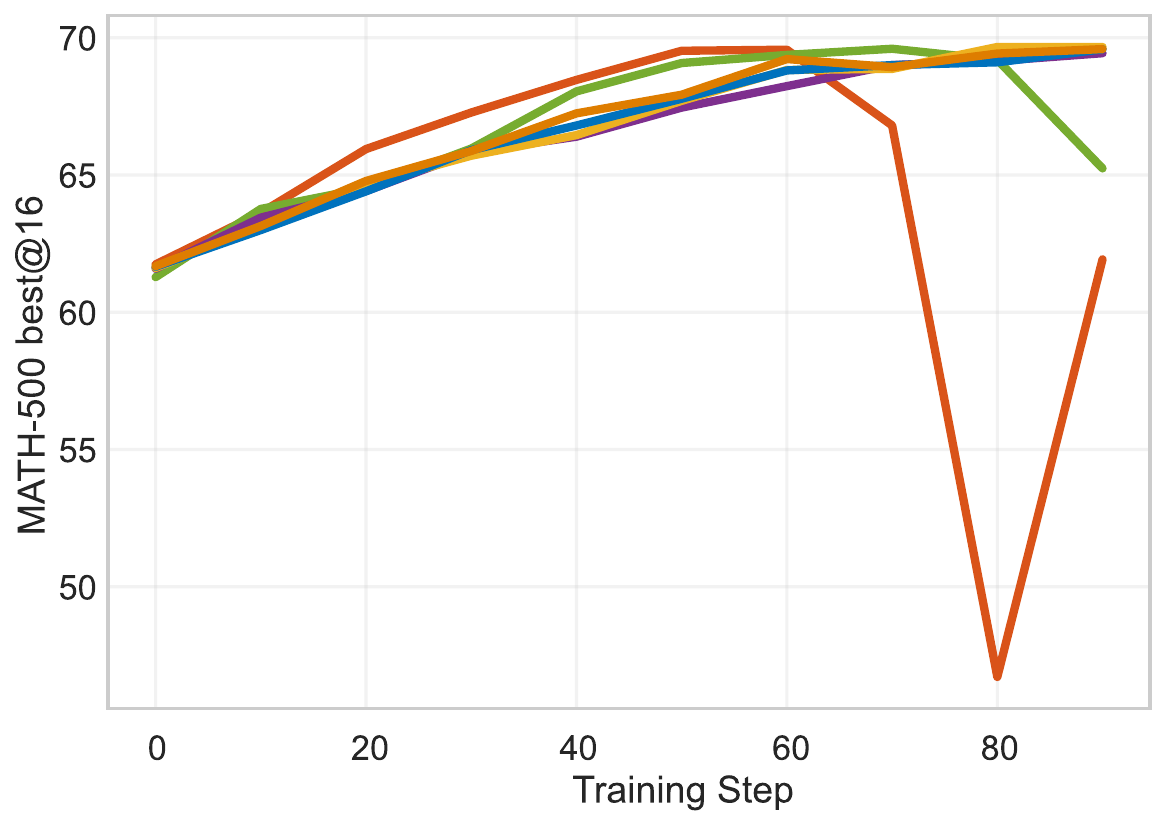}
    \caption{Drastic Best@16 Drop}
     \label{fig: off_policy_best_at_16}
  \end{subfigure}
  \begin{subfigure}[t]{0.1\textwidth}
    \centering
    \vspace{-2.5cm}
    \includegraphics[width=\linewidth]{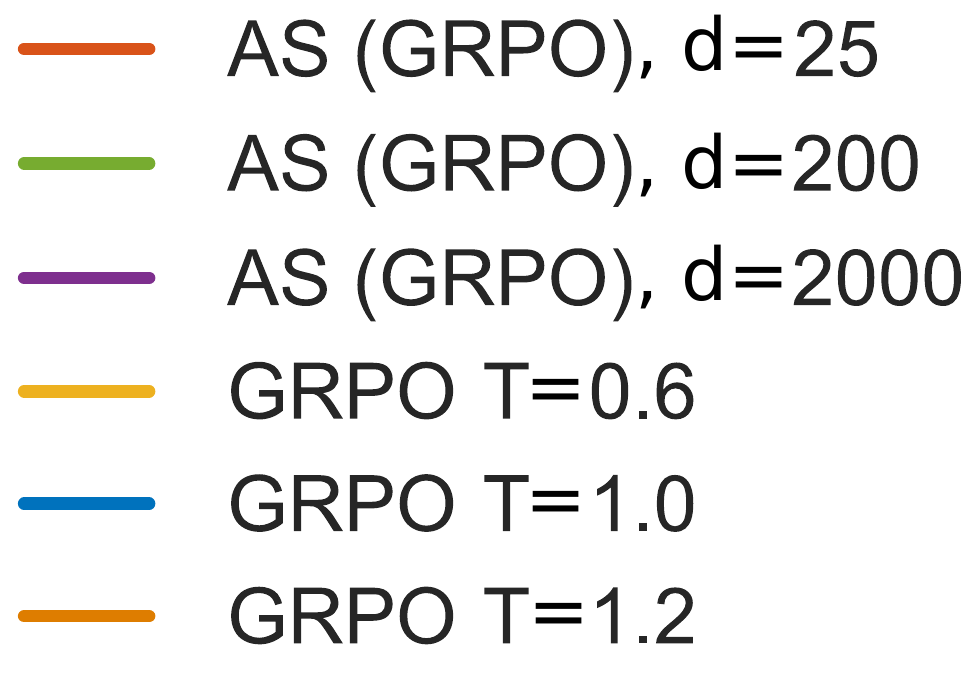}
  \end{subfigure}
  \caption{Off-policy samples bring training instability. The base model is Qwen2.5-Math-1.5B.}
  \label{fig:optim_failure_for_as_offpolicy}
\end{figure}

Noted that this off-policy phenomenon widely exists for any efficient sampling framework \citep{yao2025offpolicy} and sampling strategy (for instance, when applying best-of-$n$ sampling \citep{xiong2025minimalist} or filtering out responses \citep{shrivastava2025sample}, the underlying distribution of responses is implicitly altered). To be more precise, we take the policy gradient loss as an example:
\begin{align*}
\E_{x \sim \mathcal{D},y{\sim} \mathbin{\color{red}\pi^{\text{sampling}}_{\theta_{\text{old}}}(\cdot\mid x)}}
\left[\frac{\pi_{\theta}(y\mid x)}{\pi_{\theta_{\text{old}}}(y\mid x)}A(y;x)\right] 
=\E_{x \sim \mathcal{D},y{\sim} \mathbin{\pi_{\theta_{\text{old}}}(\cdot\mid x)}}\left[{\color{red}\frac{\pi^{\text{sampling}}_{\theta_{\text{old}}}(y\mid x)}{\pi_{\theta_{\text{old}}}(y\mid x)}}\times\frac{\pi_{\theta}(y\mid x)}{\pi_{\theta_{\text{old}}}(y\mid x)}A(y;x)\right],
\end{align*}
where $\pi^{\text{sampling}}_{\theta_{\text{old}}}(\cdot\mid x)$ represents the underlying sampling distribution. 
In such case, an extra weight is implicitly added to each response in addition to its advantage $A(y;x)$.

This extra weight can significantly inflate the variance of the policy gradient, posing a stability challenge that our proposed \alg needs to mitigate.
To quantify this effect in our proposed \alg, we now analyze such variance under a standard, fixed \textbf{temperature sampling}.

We use $\tau=1$ to define a policy $\pi$ and consider the effect of $\tau$ on the variance of the gradient estimator. We reduce the problem to analyzing the variance inflation factor
\begin{equation}
\label{eq:variance-inflation-factor}
\mathbb{E}_{y\sim\pi(\cdot\mid x;\tau)}\left[\frac{\pi(y\mid x;1)^2}{\pi(y\mid x;\tau)^2}\right].
\end{equation}
We begin with one-token case. Let $o_i$ denote the $i$th token in the vocabulary $V$ and $h_i$ is its logit. Then \eqref{eq:variance-inflation-factor} can be rewritten as
\begin{align*}
    \sum_{i=1}^{|V|}\frac{h_i/(\sum_{j=1}^{|V|} h_j)}{h^{1/\tau}_i/(\sum_{j=1}^{|V|} h^{1/\tau}_j)}\times\frac{h_i}{\sum_{j=1}^{|V|} h_j}
    = \frac{\sum_{i=1}^{|V|} h_i^{2-1/\tau}\sum_{i=1}^{|V|} h^{1/\tau}_i}{\left(\sum_{i=1}^{|V|} h_i\right)^2}
\end{align*}

\begin{proposition}
Suppose $h_i\in[0,1]$ for all $i\in V$.
$\sum_{i=1}^{|V|} h_i^{2-1/\tau}\sum_{i=1}^{|V|} h^{1/\tau}_i$ is decreasing when $\tau\le1$ and increasing when $\tau\ge1$, which implies it has a global minimum at $\tau=1$.
\end{proposition}
\begin{proof}
Let $x=1/\tau$. We define 
\[
f(x)=\log\left(\sum_{i=1}^{|V|} h_i^{2-x}\right)+\log\left(\sum_{i=1}^{|V|} h^{x}_i\right).
\]
Its derivative is
\[
f'(x)=\frac{-\sum_{i=1}^{|V|} h_i^{2-x}\log h_i}{\sum_{i=1}^{|V|} h_i^{2-x}}
+\frac{\sum_{i=1}^{|V|} h_i^{x}\log h_i}{\sum_{i=1}^{|V|} h_i^{x}}.
\]
To analyze the sign of $f'(x)$, we define a helper function $g(x) = \frac{\sum_{i=1}^{|V|} h_i^{x}\log h_i}{\sum_{i=1}^{|V|} h_i^{x}}$. Then, $f'(x)=g(x)-g(2-x)$ and its sign depends on whether $g(x)$ is greater than, less than, or equal to $g(2-x)$. We take a look at derivative of $g$:
\[
g'(x) = \frac{\left(\sum_{i=1}^{|V|} h_i^{x}(\log h_i)^2\right) \left(\sum_{i=1}^{|V|} h_i^{x}\right) - \left(\sum_{i=1}^{|V|} h_i^{x}\log h_i\right)^2}{\left(\sum_{i=1}^{|V|} h_i^{x}\right)^2}\ge0.
\]
Hence, $g$ is an increasing function and
\begin{equation*}
\begin{cases}
~f'(x)=g(x)-g(2-x)\ge0,~~\mathrm{when}~x\ge1\\
~f'(x)=g(x)-g(2-x)=0,~~\mathrm{when}~x=1\\
~f'(x)=g(x)-g(2-x)\le0,~~\mathrm{when}~x\le1.
\end{cases}
\end{equation*}
Accordingly, $f$ is increasing when $x\ge1$ and is decreasing when $x\le1$. Then the proposition easily follows.
\end{proof}

The same conclusion can be proved for multiple-token sequence by induction. Therefore, we get that
when the temperature is far from 1, the off-policy issue could be severe and lead to large variance of the gradient estimator. 

\subsection{Truncated Importance Sampling Ratio Correction}
\label{subsec:tis-correction}
To cancel such bias, an importance sampling ratio can be introduced~\citep{hilton2022batch,yao2025offpolicy}:
\begin{align*}
\E_{x \sim \mathcal{D},y{\sim} \mathbin{\pi_{\theta_{\text{old}}}(\cdot\mid x)}}\left[\frac{\pi_{\theta}(y\mid x)}{\pi_{\theta_{\text{old}}}(y\mid x)}A(y;x)\right]
=
\E_{x \sim \mathcal{D},y{\sim} \mathbin{\color{red}\pi^{\text{sampling}}_{\theta_{\text{old}}}(\cdot\mid x)}}
\left[{\color{red}\frac{\pi_{\theta_{\text{old}}}(y\mid x)}{\pi^{\text{sampling}}_{\theta_{\text{old}}}(y\mid x)}}\times\frac{\pi_{\theta}(y\mid x)}{\pi_{\theta_{\text{old}}}(y\mid x)}A(y;x)\right]. 
\end{align*}
To further prevent negative effects by the extreme likelihood ratios and boost training stability, we truncate the likelihood ratio with an upper bound. That is, \emph{truncated importance sampling} technique \citep{heckman1998matching}. Taking the vanilla policy gradient loss as an example, the modified loss for EAD is as follows:
\begin{equation*}
\label{eq:tis-pg}
\E_{x \sim \mathcal{D},y{\sim}{\pi^{\text{EAD}}_{\theta_{\text{old}}}(\cdot\mid x)}}
\left[{\color{blue}\min\left(\frac{\pi_{\theta_{\text{old}}}(y\mid x)}{\pi^{\text{EAD}}_{\theta_{\text{old}}}(y\mid x)},\varepsilon\right)}\frac{\pi_{\theta}(y\mid x)}{\pi_{\theta_{\text{old}}}(y\mid x)}A(y;x)\right].
\end{equation*}

\newpage
\section{Proof of Temperature-Entropy Relationship}
\label{app:entropy_proof}

\begin{proposition}
The entropy of the softmax distribution is a non-decreasing function of the temperature $\tau > 0$.
\end{proposition}

\begin{proof}
The strategy is to show that the entropy $H$ is a non-increasing function of the inverse temperature $\beta = 1/\tau > 0$.
The probability of sampling token $v$ with temperature $\tau$ is given by the policy $\pi_\theta$. For simplicity in the derivation, we denote this probability as $p_v(\tau)$:
\[
p_v(\tau) \triangleq \pi_\theta(y_t=v \mid [x, y_{<t}]; \tau)
\]

Let $h_v$ be the logit for a token $v$ in the vocabulary $V$. The probability of a token as a function of $\beta$ is given by:
\[
p_v(\beta) = \frac{\exp(\beta h_v)}{\sum_{v' \in V} \exp(\beta h_{v'})} \triangleq \frac{\exp(\beta h_v)}{Z(\beta)},
\]
where $Z(\beta)$ is the partition function. The entropy, as a function of $\beta$, is:
\[
H(\beta) = -\sum_{v \in V} p_v(\beta) \log p_v(\beta).
\]
We can rewrite the entropy by substituting $\log p_v(\beta) = \beta h_v - \log Z(\beta)$:
\begin{align*}
H(\beta) &= -\sum_{v \in V} p_v(\beta) (\beta h_v - \log Z(\beta)) \\
&= \log Z(\beta) \left(\sum_{v \in V} p_v(\beta)\right) - \beta \sum_{v \in V} h_v p_v(\beta) \\
&= \log Z(\beta) - \beta \cdot \mathbb{E}_{v \sim p(\beta)}[h_v].
\end{align*}
Now, we differentiate $H(\beta)$ with respect to $\beta$. Let $\bar{h}(\beta) = \mathbb{E}[h_v]$.
\[
\frac{dH}{d\beta} = \frac{d}{d\beta}(\log Z(\beta)) - \frac{d}{d\beta}(\beta \bar{h}(\beta)).
\]
First, we find the derivative of the log-partition function:
\[
\frac{d}{d\beta}(\log Z(\beta)) = \frac{Z'(\beta)}{Z(\beta)} = \frac{\sum_v h_v \exp(\beta h_v)}{Z(\beta)} = \sum_v h_v p_v(\beta) = \bar{h}(\beta).
\]
Next, we use the product rule for the second term:
\[
\frac{d}{d\beta}(\beta \bar{h}(\beta)) = \bar{h}(\beta) + \beta \frac{d\bar{h}}{d\beta}.
\]
Combining these gives:
\[
\frac{dH}{d\beta} = \bar{h}(\beta) - \left(\bar{h}(\beta) + \beta \frac{d\bar{h}}{d\beta}\right) = -\beta \frac{d\bar{h}}{d\beta}.
\]
The derivative $\frac{d\bar{h}}{d\beta}$ is the variance of the logits. We can show this by differentiating $\bar{h}(\beta)$:
\begin{align*}
\frac{d\bar{h}}{d\beta} &= \frac{d}{d\beta}\left( \frac{\sum_v h_v \exp(\beta h_v)}{Z(\beta)} \right) \\
&= \frac{(\sum_v h_v^2 \exp(\beta h_v))Z(\beta) - (\sum_v h_v \exp(\beta h_v))Z'(\beta)}{Z(\beta)^2} \\
&= \sum_v h_v^2 p_v(\beta) - \left(\sum_v h_v p_v(\beta)\right)\left(\frac{Z'(\beta)}{Z(\beta)}\right) \\
&= \mathbb{E}[h^2] - (\mathbb{E}[h])^2 = \mathrm{Var}_{v \sim p(\beta)}(h_v).
\end{align*}
Substituting this back, we arrive at the final expression for the derivative of entropy:
\[
\frac{dH}{d\beta} = -\beta \cdot \mathrm{Var}_{v \sim p(\beta)}(h_v).
\]
By definition, the temperature $\tau > 0$, so the inverse temperature $\beta > 0$. The variance of any random variable is non-negative. This can be formally shown using \textbf{Jensen's inequality}: for the convex function $\phi(x)=x^2$, we have $\mathbb{E}[\phi(h)] \ge \phi(\mathbb{E}[h])$, which means $\mathbb{E}[h^2] \ge (\mathbb{E}[h])^2$, and thus $\mathrm{Var}(h) \ge 0$.

Therefore, the derivative of entropy with respect to $\beta$ is non-positive:
\[
\frac{dH}{d\beta} = \underbrace{-\beta}_{\le 0} \cdot \underbrace{\mathrm{Var}(h_v)}_{\ge 0} \le 0.
\]
Since $H(\beta)$ is a non-increasing function of $\beta$, and $\beta$ is inversely proportional to $T$, it follows that $H(\tau)$ must be a non-decreasing function of the temperature $\tau$.
\end{proof}

\section{Increasing Length in RL Training}
\label{app: increased_length}
During RL training, our algorithm (\alg) incentivizes the model to generate longer, more effective reasoning chains for difficult problems, especially for 7B models (\cref{fig: response_length}). While both \alg and temperature sampling initially learn to shorten their responses by shifting from complex code-based solutions to direct mathematical reasoning, their behavior later diverges. The response length from temperature sampling stabilizes, whereas \alg learns to selectively increase reasoning length for harder problems, which boosts final performance. For these experiments, EAD is applied in the DAPO algorithm for sampling rollouts. We use the same training setup as detailed in \cref{sec:experiments}. 

\begin{figure}[h!]
\centering
\includegraphics[width=\linewidth]{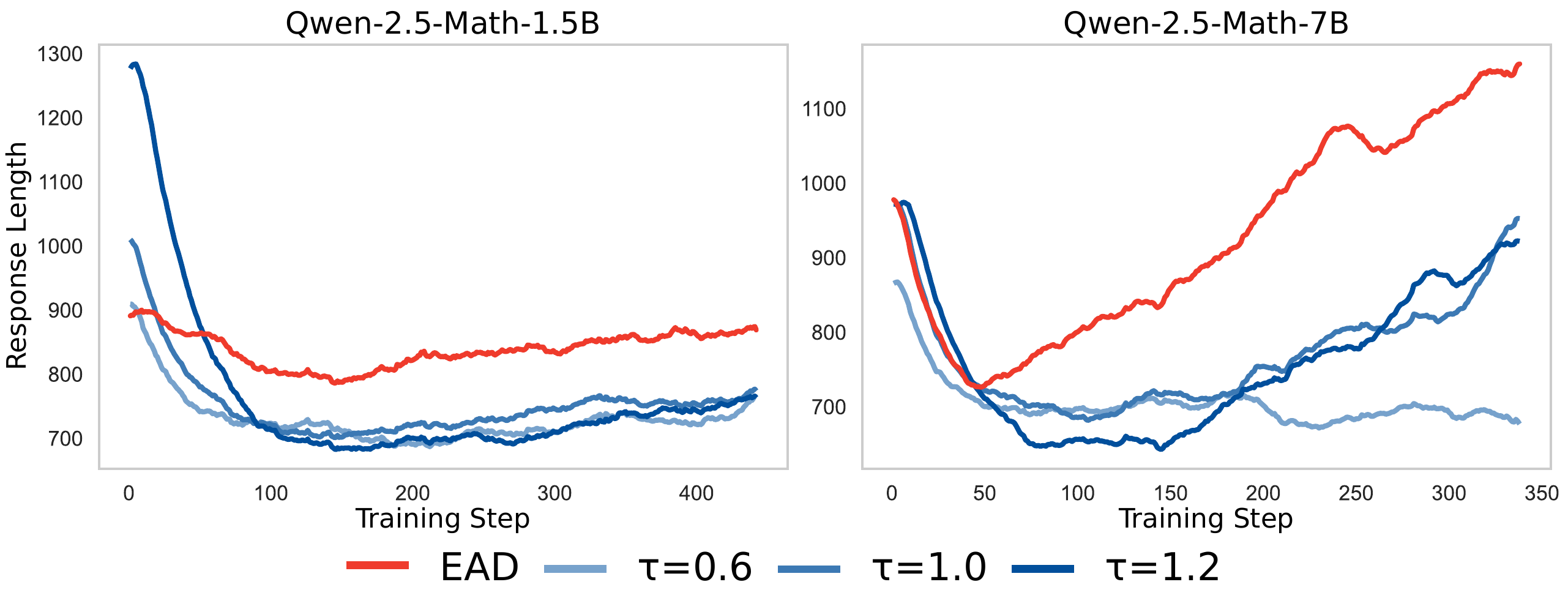}
\caption{Compared with normal temperature sampling, EAD can naturally incentivize the model to generate longer reasoning chains. }
\label{fig: response_length}
\end{figure}

\end{document}